\documentclass[twoside,11pt]{article}
\usepackage{jair, theapa, rawfonts}

\usepackage{adjustbox}
\usepackage{amsmath,amsfonts}
\usepackage{amsthm}
\usepackage{amssymb}
\usepackage[linesnumbered, ruled, vlined]{algorithm2e}
\usepackage{array}
\usepackage{enumitem}
\usepackage[T1]{fontenc}
\usepackage{graphicx}
\usepackage{multicol}
\usepackage{multirow}
\usepackage{relsize}
\usepackage{rotating}
\usepackage{stfloats}
\usepackage{subcaption}
\usepackage{tabularx}
\usepackage{tcolorbox}
\usepackage{textcomp}
\usepackage{url}
\usepackage{verbatim}
\usepackage{xcolor}
\usepackage{colortbl}

\SetKwFor{While}{while}{}{end}
\SetKwFor{ForEach}{for each}{}{end}
\SetKwIF{If}{ElseIf}{Else}{if}{ }{else if}{else}{end if}
\SetKwInput{KwData}{Input}
\SetKwInput{KwResult}{Output}

\definecolor{customcolor}{HTML}{2a45f7} 

\newtheorem{assumption}{Assumption}
\newtheorem{definition}{Definition}
\newtheorem{lemma}{Lemma}
\newtheorem{theorem}{Theorem}
\newtheorem{property}{Property}

\newenvironment{implication}[1][Key Implication]{%
    \noindent%
    \normalfont\textit{#1. }%
}{%
    \endtrivlist\medskip
}

\jairheading{}{2025}{}{2/25}{}
\ShortHeadings{PRISM: Complete Online Decentralized MAPF}
{Lee, Serlin, Motes, Long, Morales, \& Amato}
\firstpageno{25}

\begin{document}

\title{PRISM: Complete Online Decentralized Multi-Agent Pathfinding with Rapid Information Sharing \\  using Motion Constraints}

\author{\name Hannah Lee \email hannah9@illinois.edu \\
       \addr University of Illinois at Urbana-Champaign, 201 N Goodwin Ave,  \\Urbana, IL 61801 USA\\
       \AND
       \name Zachary Serlin \email zachary.serlin@ll.mit.edu \\
       \addr MIT Lincoln Laboratory, 244 Wood St, \\ Lexington, MA 02421 USA \\
       \AND
       \name James Motes \email jmotes2@illinois.edu \\
       \addr University of Illinois at Urbana-Champaign, 201 N Goodwin Ave,  \\Urbana, IL 61801 USA\\
       \AND
       \name Brendan Long \email brendan.long@ll.mit.edu\\
       \addr MIT Lincoln Laboratory, 244 Wood St, \\ Lexington, MA 02421 USA \\       
       \AND
       \name Marco Morales\email moralesa@illinois.edu \\
       \name Nancy M. Amato\email namato@illinois.edu \\
       \addr University of Illinois at Urbana-Champaign, 201 N Goodwin Ave,  \\Urbana, IL 61801 USA\\
       }


\maketitle

\begin{abstract}
We introduce PRISM (Pathfinding with Rapid Information Sharing using Motion Constraints), a novel decentralized algorithm designed to address the multi-task multi-agent pathfinding (MT-MAPF) challenge. PRISM enables large teams of agents to concurrently plan safe and efficient paths for multiple tasks while avoiding collisions. It employs a rapid communication strategy that uses information packets to exchange motion constraint information, enhancing cooperative pathfinding and situational awareness, even in scenarios without direct communication. We theoretically prove that PRISM resolves and avoids all deadlock scenarios when possible, a critical hurdle in decentralized pathfinding. Empirically, we evaluate PRISM across five environments and 25 random scenarios, benchmarking it against the centralized Conflict-Based Search (CBS) and the decentralized Token Passing with Task Swaps (TPTS) algorithms. PRISM demonstrates exceptional scalability and solution quality, supporting 3.4 times more agents than CBS and handling up to 2.5 times more tasks in narrow passage environments than TPTS. Additionally, PRISM matches CBS in solution quality while achieving faster computation times, even under low-connectivity conditions. Its decentralized design significantly reduces the computational burden on individual agents, making it highly scalable for large-scale environments. These results confirm PRISM’s robustness, scalability, and effectiveness in addressing complex and dynamic pathfinding scenarios.
\end{abstract}

\section{Introduction} \label{Introduction}

Multi-agent pathfinding (MAPF) is an area of research concerned with the coordination of multiple autonomous agents as they navigate from individual starting points to designated goals. This problem is critical for robots to operate efficiently and safely in shared spaces. The main challenge is not just finding efficient paths for individual agents but coordinating these paths to prevent conflicts while optimizing criteria like total time or path length. This is relevant in applications such as assembly \shortcite{hlw-agffaptmsa-98,nb-toaraap-93,bpssk-ostapffmarap-20}, evacuation \cite{ra-bbep-10}, formation \shortcite{ba-bbfcfmt-98,tpk-ltfs-04,kh-ppfpimf-06,lsmfkk-maifice-20,lwcycl-mfifadhlatmamt-21}, localization \cite{fbkt-apptcmrl-00}, and object transportation\shortcite{bpssk-ostapffmarap-20,rdjj-mfwtofar-95}.

The complexity of MAPF stems from its high computational demands and the need for scalable solutions in real-world scenarios. As the number of agents increases, the space of potential interactions and conflicts grows exponentially, making traditional pathfinding algorithms like A* inadequate \cite{yl-saioomrppog-13}. These algorithms also struggle with dynamic settings where obstacles and agent goals can change in real time. Moreover, many existing solutions rely on unrealistic assumptions, such as centralized coordination, perfect communication, or single-task assignments per agent, limiting their applicability in practical settings \cite{salzman2020research}. 

This complexity is further amplified when multiple objectives are introduced. Solvers must address not only optimal pathfinding but also efficient task sequencing and allocation, which can lead to a combinatorial explosion in solution space. These heightened computational demands underscore the limitations of centralized approaches and highlight the practicality of decentralized methods in environments requiring adaptability and real-time decision-making. 

In response to these challenges, decentralized solvers offer a promising alternative by enabling agents to compute their paths independently while dynamically resolving conflicts through local interactions. This decentralized autonomy reduces reliance on a central coordinator, making it particularly advantageous in environments with unreliable or constrained communication. Decentralized methods also exhibit enhanced adaptability, allowing agents to adjust their paths in response to environmental changes and evolving objectives in real time. Such flexibility enhances system resilience, especially in highly dynamic conditions or scenarios involving failures and malicious interference, where centralized approaches often fall short. 

Despite their promise, decentralized MAPF solvers face notable challenges. Achieving scalability often comes at the cost of completeness guarantees, leaving larger systems vulnerable to deadlock. Many solvers also operate under restrictive assumptions, such as unrestricted communication among agents, which are often unrealistic in large-scale or resource-constrained environments. Others attempt to mitigate these issues by relying on predefined task designs, sidestepping the need for higher-level coordination. These limitations underscore the need for a decentralized MAPF framework that combines scalability and high solution quality while adhering to realistic operation constraints, motivating the development of our proposed approach. 

In this paper, we address the multi-task multi-agent pathfinding (MT-MAPF) problem, which combines task allocation and pathfinding into a cohesive framework. Task allocation assigns objectives to agents based on factors like proximity and path length, forming the foundation for conflict-free pathfinding. Efficient task allocation depends on accurate path estimates, while successful pathfinding relies on well-distributed tasks. These components must operate in tandem, with each phase informing the other to maximize overall system performance. 

We propose PRISM (Pathfinding with Rapid Information Sharing using Motion constraints), a novel online decentralized MAPF algorithm designed to solve the MT-MAPF problem under constrained communication protocols, where only a subset of agents can communicate at any time. The key contributions of PRISM are as follows: 
\begin{enumerate}
    \item \textbf{Completeness} The algorithm provably resolves all solvable deadlock scenarios, guaranteeing task completion for all agents. 
    \item \textbf{Flexible online decentralized planner:} PRISM ensures scalability, robustness, and adaptability in dynamic environments where tasks and team specifications may evolve, all while maintaining high solution quality.
    \item \textbf{Constrained communication protocols:} PRISM effectively enables decentralized pathfinding by utilizing info  packets with motion constraints, facilitating safe and efficient coordination among agents.
    \item \textbf{Improved scalability and robustness:} Empirical evaluations show PRISM maintaining high solution quality with up to 30 agents and 575 tasks, outperforming centralized methods in efficiency while achieving comparable solution quality. Additionally, PRISM surpasses existing decentralized approaches in reliability, highlighting its robustness and adaptability to dynamic real-world applications. 
\end{enumerate}

\section{Problem Formulation} \label{ProblemFormulation}

In Multi-Task Multi-Agent Pathfinding (MT-MAPF), $n$ agents, denoted as $R = \{R_1, \dots, R_n\}$, navigate a shared environment to complete $m$ tasks, represented by $T = \{T_1^?, \dots, T_m^?\}$, with `?' indicating an unassigned task. These agents possess limited communication capabilities, restricted to those within range, yet they can extend their reach through multi-hop communication within their network. In a system of $|R|$ agents, a local network is assumed to consist of a subset of agents that can communicate with each other through multiple hops. Consequently, all agents within the same connected component are considered part of the same local network. 

Agents start with complete knowledge of the environment but lack information about other agents, including their objectives and locations. As agents move along their designated paths, they exchange information upon entering communication range, enabling collaborative adjustments to their paths to avoid conflicts.

Tasks are initially unassigned and represented as tuples $T_j^? = \langle s_j, g_j \rangle^?$, where the superscript $?$ denotes their unassigned state. Each task is defined by a starting position $s_j$ and a goal position $g_j$.  The task allocator assigns these tasks to agents, ensuring that each task is undertaken by an appropriate agent, such that $T_j^i$ denotes task $j$ has been assigned to agent $R_i$. We classify tasks as either mission tasks or transition tasks: mission tasks are elements of $T$, while transition tasks enable movement between mission tasks. 

We assume that the MT-MAPF problem consists of solvable mission tasks. This implies the existence of a task allocation where no two tasks share the same start or goal position, and the task allocator is capable of identifying such an allocation. As a result, no two agents will simultaneously attempt to access the same endpoint as a final resource. As a result, all goal positions serve as safe resting locations for agents, ensuring they do not permanently obstruct others from reaching their goals. 

The task allocator is responsible for assigning unstarted mission tasks to agents. Once an agent begins a mission task, it cannot be reassigned to another agent. However, if an agent is transitioning to a mission, the mission task can be reassigned to another agent if necessary. In cases where the original agent is not assigned a new mission task after the reassignment, it is expected to return to the goal position of its previous mission task. 

Agents operate within an undirected graph $G = (V, E) $ which models the environment as a two-dimensional grid world where movement is constrained to the four cardinal directions. Vertices $V$ represent feasible positions within the environment and edges $E$ enable transitions between adjacent vertices. Agents progress through the environment in discrete timesteps, choosing at each timestep to either move to a connected vertex or remain stationary. 

The objective of MT-MAPF is to develop a coordinated team plan  $\Pi = \{\pi_1, \dots, \pi_n\}$, where the plan for each agent, $\pi_i = [p_1, \dots, p_k]$, consists of concatenated paths designed to complete assigned mission tasks and transition tasks. Each path in the agent's plan is a sequence of positions $p_k = (v, \dots, v')$. The position of an agent at any given timestep $t$ is determined by $\pi_i[t]$. An agent's plan length is defined as $|\pi_i|$ which is the sum of all of its individual paths $\sum_{j = 1}^k|p_j|$. 

Agent paths must be planned to avoid conflicts. Vertex conflicts are denoted as  $\langle R_i, R_j, v, t \rangle$ and occur when agents $R_i$ and $R_j$ occupy the same vertex $v$ at the same time $t$. Edge conflicts, denoted as $\langle R_i, R_j, v, v', t, t' \rangle$, occur when agents $R_i$ and $R_j$ attempt to traverse the same edge $\langle v, v' \rangle$ between times $(t, t')$.  

MT-MAPF aims to produce a collision-free team plan that efficiently completes all tasks while minimizing the sum-of-costs objective. Sum-of-costs is defined as the total cost of all plans and this objective is expressed as follows: 
\begin{align*}
    Sum\ of\ Costs(\Pi) = \sum_{i = 1}^n |\pi_i|
\end{align*}

\section{Related Work} \label{Related Work}

Centralized approaches to MAPF rely on a single entity, with comprehensive knowledge of the environment and agents' states, that can compute optimal or near-optimal paths for all agents simultaneously, often resulting in efficient conflict resolution and high-quality solutions. However, these algorithms face significant scalability challenges as the number of agents increases, owing to the exponential growth of the state space and the computational complexity required to coordinate interactions across agents. Furthermore, centralized methods are typically offline and are less adaptable to dynamic environments where real-time decision-making is required.

In contrast, decentralized approaches distribute decision-making among agents, enabling each to plan its path using local information and limited communication. This decentralized structure enhances scalability and robustness, particularly in dynamic environments, as agents can quickly adapt to environmental changes and the behavior of other agents. However, the lack of centralized oversight poses challenges in achieving globally optimal solutions and ensuring effective coordination, especially in densely populated or highly constrained scenarios. One critical issue in decentralized MAPF is the potential for deadlocks, where agents become indefinitely stalled due to mutually conflicting paths or resource contention. Resolving or avoiding deadlocks, which often requires sophisticated coordination mechanisms, can be difficult to achieve without centralized control or under constrained communication conditions. Consequently, many decentralized algorithms operate under narrowly defined assumptions, which can limit their applicability and generalizability to broader or more complex settings.


Another fundamental distinction in MAPF lies between online and offline planning paradigms. Offline planning assumes that all relevant information about the environment and agents is available prior to computation, allowing for optimized, precomputed paths. While this approach can yield high-quality solutions, it is often impractical in dynamic or partially observable environments where unforeseen changes or disruptions occur. In contrast, online planning operates in real time, with agents continually updating their plans based on new information as it becomes available. Although online methods are well-suited for dynamic and uncertain scenarios, they face significant challenges in maintaining solution quality and avoiding conflicts due to limited time for computation and communication. These trade-offs between computational efficiency, adaptability, and solution quality are central to the development of effective decentralized MAPF algorithms.

To address these coordination challenges, the field offers a diverse range of approaches, including search-based solvers, partition-based solvers, priority-based solvers, consensus-based solvers, rule-based methods, and potential fields. Each method involves distinct trade-offs in scalability, solution quality, and adaptability. This section provides a concise overview of these approaches, emphasizing their respective strengths and limitations, and positions our contributions through PRISM within this context. Table \ref{table:related} summarizes the strengths and limitations of these methods, with tildes indicating areas where some research progress has been made but remains limited. Additionally, we introduce Conflict-Based Search (CBS) \shortcite{cbs}, a constraint-based algorithm that serves as a foundational framework for PRISM.

\begin{table*}[!t]
\centering
\caption{Strengths and Limitations of General Solvers}
{
\resizebox{\textwidth}{!}{%
\begin{tabular}{|l|c|c|c|c|c|c|c|}
\hline
\rowcolor[HTML]{EFEFEF} 
\multicolumn{1}{|c|}{\cellcolor[HTML]{EFEFEF}Feature} & PRISM & Search-Based & Partition-Based & Priority-Based & Rule-Based & Potential Field & Consensus-Based \\ \hline
Centralized & \textbf{} & \textbf{$\times$} & \textbf{} & \textbf{$\times$} & \textbf{} & \textbf{} & \textbf{} \\ \hline
Decentralized & \textbf{$\times$} & \textbf{} & \textbf{$\times$} & \textbf{$\times$} & \textbf{$\times$} & \textbf{$\times$} & \textbf{$\times$} \\ \hline
Offline Planning & \textbf{} & \textbf{$\times$} & \textbf{} & \textbf{$\times$} & \textbf{} & \textbf{} & \textbf{} \\ \hline
Online Planning & \textbf{$\times$} & \textbf{} & \textbf{$\times$} & \textbf{$\times$} & \textbf{$\times$} & \textbf{$\times$} & \textbf{$\times$} \\ \hline
Complete & \textbf{$\times$} & \textbf{$\times$} & \textbf{} & \textbf{} & \textbf{} & \textbf{} & \textbf{} \\ \hline
Resolves Deadlock & \textbf{$\times$} & \textbf{$\times$} & \textbf{$\sim$} & \textbf{$\sim$} & \textbf{} & \textbf{} & \textbf{$\sim$} \\ \hline
Constrained Comms & \textbf{$\times$} & \textbf{} & \textbf{$\times$} & \textbf{} & \textbf{$\times$} & \textbf{$\times$} & \textbf{$\times$} \\ \hline
Peer-to-Peer Comms & \textbf{} & \textbf{} & \textbf{} & \textbf{$\times$} & \textbf{} & \textbf{} & \textbf{} \\ \hline
\end{tabular}%
}
}
\label{table:related}
\end{table*}

\subsection{Search-Based Solvers} 

Search-based solvers form a foundational category of algorithms for centralized, offline MAPF. These methods systematically explore the state space to find solutions, often employing graph search techniques such as A*, its derivatives, or other combinatorial optimization strategies. By exhaustively analyzing all possible configurations of agent paths, search-based solvers are capable of providing high-quality solutions, frequently achieving optimal or near-optimal paths for agents \shortcite{cbs,hccbs,s-cp-05,pbs,wc-sefmrpp-15,cukk-oabsmamp-19,smsa-romrmpucbs-21,ma2016multi,lmlk-tappfmapd-19,bfsssbt-icbsfomapf-15,lhsfmk-dsfmapfwcbs-19,lhsmk-sbcfgbmapf-19,lghsmk-ntfpsbimapf-20}. A notable example is Conflict-Based Search (CBS), which uses a hierarchical strategy to decompose complex problems into smaller subproblems, efficiently resolving conflicts among agents.
\
Constraint-based search algorithms, such as \shortcite{cbs,hccbs,pbs,bfsssbt-icbsfomapf-15,lhsfmk-dsfmapfwcbs-19,lhsmk-sbcfgbmapf-19,lghsmk-ntfpsbimapf-20}, represent a powerful subcategory of search-based solvers, focusing on systematically managing conflicts among agents to ensure feasibility and efficiency. These methods operate by identifying conflicts, such as two agents attempting to occupy the same space at the same time, and resolving them through the addition of constraints that guide future search iterations. This iterative refinement enables constraint-based algorithms to balance solution quality and computational efficiency effectively. While highly capable of handling complex coordination problems, their reliance on centralized, offline computation limits their applicability in dynamic or large-scale environments where constraints and objectives frequently change.

Despite their advantages, these search-based solvers face significant scalability challenges as the number of agents increases, leading to exponential growth in the computational complexity required to coordinate interactions across agents. This issue is compounded in settings with large-scale agent teams or highly constrained environments. Additionally, offline, centralized solvers assume a priori knowledge of all tasks and environmental conditions, limiting their ability to adapt to dynamic changes or unexpected obstacles. These solvers typically do not address the distribution of computational workload, which becomes a bottleneck in scenarios involving a higher number of tasks than available agents.

\subsection{Partition-Based Solvers}
Partition-based solvers address scalability challenges in MAPF by dividing the environment into smaller regions and solving subproblems within each partition. These approaches are typically online, enhancing scalability and reducing computational overhead, which makes them well-suited for larger environments where exhaustive global coordination is impractical. Partition-based solvers avoid relying on complete peer-to-peer communication across the entire agent team, enabling more effective handling of large-scale scenarios. However, they often struggle with coordinating agents at partition boundaries, potentially leading to suboptimal solutions.

Partition-based solvers can be broadly categorized into those that employ controlling agents or coordinators for each partition and those that allocate partitions on a per-agent basis to prevent collisions. In the first category, methods such as \shortcite{wilt2014spatially,pianpak2019distributed}, assume the presence of a central coordinator agent within each partition, responsible for managing the activities and paths of all agents in its assigned region. Coordination between neighboring partitions is achieved by facilitating agent transfers across boundaries, typically managed through interactions between the respective coordinators. This approach simplifies intra-partition coordination but depends on the assignment of specific roles to agents, which may not always align with the capabilities or distribution of the team.

In contrast, partition-based solvers that operate on a per-agent basis, such as \shortcite{purwin2008theory,gui2023decentralized}, aim to prevent collisions by assigning each agent to a separate partition, minimizing direct agent-to-agent interaction. These algorithms often require conservative planning strategies to ensure safe operation, particularly when partitions overlap or agents approach boundary regions. While this method eliminates the need for explicit role assignments, its conservative nature can limit overall system efficiency and the quality of solutions.

\subsection{Priority-Based Solvers}
Priority-based solvers assign priorities to agents and plan their paths sequentially, offering a practical approach for MAPF by reducing the complexity of simultaneous planning. These methods are often decentralized, making them suitable for distributed systems, but can also be implemented in centralized frameworks \shortcite{pbs,chan2023greedy}. A key trade-off of priority-based solvers is their simplicity and scalability compared to exhaustive search methods; however, they often produce highly suboptimal paths for lower-priority agents and are prone to deadlocks in densely populated environments.

Priority-based algorithms can be broadly categorized into online and offline approaches. Offline methods, such as \shortcite{pbs,ho2020decentralized,chan2023greedy}, precompute priorities and paths before execution, reducing the need for real-time decision-making but limiting adaptability to dynamic or uncertain environments. In contrast, online methods, such as \shortcite{mlkk-lmapffopadt-17,desaraju2011decentralized,velagapudi2010decentralized}, adapt priorities and paths in real time, often employing token-passing schemes to coordinate agents. For instance, \shortcite{mlkk-lmapffopadt-17} and \shortcite{desaraju2011decentralized} use a token-based priority system to allow dynamic task allocation, while \shortcite{velagapudi2010decentralized} provides a solution by dynamically adjusting priorities and handling agent interactions during execution.

Many decentralized priority-based solvers lack strong guarantees for avoiding or resolving deadlocks, particularly in highly constrained or densely populated environments. To address this, some solvers operate under the assumption that the problem is well-formed, meaning endpoints (start or goal locations) are distributed such that no single agent can block access between other endpoints. Under this assumption, these algorithms can guarantee completeness by ensuring that a path always exists without traversing another agent's endpoints. While this condition is straightforward to enforce when designing tasks, it is challenging to verify for an arbitrary task set. For instance, \shortcite{mlkk-lmapffopadt-17,vcap2015prioritized}  leverage this assumption to handle deadlock scenarios effectively but struggle in environments where this condition does not hold. These algorithms also commonly require full peer-to-peer communication across the team, which can be a limiting factor in settings where communication is constrained or unreliable.

\subsection{Consensus-Based Solvers}

Consensus-based algorithms are commonly employed in the task allocation phase of multi-task multi-agent pathfinding problems, where agents must decide how to distribute tasks among themselves before planning individual paths. These algorithms rely on decentralized coordination mechanisms to ensure agents collectively agree on task assignments, making them particularly well-suited for scenarios with large numbers of agents and dynamic environments. By enabling distributed decision-making, consensus-based approaches enhance scalability and robustness, especially when centralized control is impractical or communication infrastructure is limited \shortcite{mikkelsen2023distributed,wang2022consensus,choi2009consensus}.

Consensus-based methods come with notable trade-offs. The iterative nature of consensus-seeking processes, where agents exchange information and refine agreements, can introduce significant delays, particularly in time-critical applications. While the agreement process is designed to avoid task conflicts and deadlocks, the reliance on initial estimates or priorities can lead to suboptimal task allocations that are difficult to revise once finalized. Limited communication ranges or sparse network connectivity exacerbate these issues, as incomplete or delayed information can result in partial or incorrect agreements, undermining the system’s overall effectiveness. Moreover, achieving consensus often requires a high degree of agent coordination, which can impose additional communication overhead and limit scalability in dense environments. Conversely, reducing the level of coordination to prioritize efficiency risks local conflicts, task redundancy, or unbalanced workloads among agents.

In addition to these drawbacks, consensus-based algorithms are not typically used for direct MAPF because their focus on achieving agreement across agents is less suited for the real-time, fine-grained conflict resolution required for pathfinding. The iterative nature of consensus processes can struggle to adapt quickly to dynamic pathfinding scenarios where agent trajectories must be recalculated frequently in response to changing conditions. This makes them more effective for high-level task allocation rather than low-level path coordination.

Overall, consensus-based algorithms excel in balancing decentralized task allocation with system-wide coordination, making them a valuable component of multi-task MAPF solutions. Their effectiveness depends on careful algorithmic design to mitigate communication overhead, adapt to dynamic conditions, and balance coordination with efficiency.

\subsection{Other Solvers}
Rule-based solvers rely on predefined protocols or algorithms that agents independently follow to avoid collisions and reach their goals. These rules often involve agents negotiating paths or prioritizing movements based on criteria such as agent IDs or proximity to target locations, enabling effective coordination without centralized control \shortcite{hwang2007protocol,izadi2011rule,asama1991collision,masehian2010hierarchical}. While efficient, the performance of rule-based solvers is highly dependent on the quality of their predefined rules. As a result, they lack adaptability to dynamic scenarios and do not provide strong guarantees for avoiding or resolving deadlocks.

Potential field methods are computationally efficient techniques that model agents as being influenced by artificial forces, such as attractive forces pulling them toward their goals and repulsive forces pushing them away from obstacles and other agents. These methods can incorporate vehicle dynamics, making them suitable for environments requiring smooth trajectories \shortcite{sigurd2003uav,shim2003decentralized,matoui2017path,xie2022distributed,pradhan2018motion}. However, potential field methods offer no guarantees on collision avoidance, especially in densely populated environments or when agents become trapped in local minima.

\subsection{Contributions of PRISM}

PRISM addresses key limitations of existing multi-agent pathfinding algorithms by providing a robust framework for online, decentralized planning. Unlike traditional approaches, it operates without restrictive assumptions, enabling application to complex environments. PRISM guarantees deadlock resolution and avoidance, even with constrained communication, while dynamically adapting to changes in team composition and tasks. Despite its decentralized nature, PRISM ensures completeness and achieves a balance between scalability and efficiency. Moreover, it extends constraint-based search methods to online, decentralized settings, preserving their strengths while enhancing adaptability.

{
\begin{algorithm}[t]
\small
\caption{Modified Conflict-Based Search}
\label{alg:cbs}
\KwData{Robots $R$, Info Packets $P$}
\KwResult{Success or Failure}

$n_0 \gets$ empty CT node\\
$CT \gets$ empty priority queue\\

\color{customcolor}
\ForEach{$R_i$ \upshape{in} $R$} {
    $n_0.cstr[R_i] \gets R_i.cstr$\\
    $n_0.plan[R_i] \gets R_i.path$\\
}

\ForEach{$p_i$ \upshape{in} $P$}{
    $n_0.cstr[R_i] \gets p_i.cstr$ \\
    $N_0.plan[R_i] \gets $\textsc{LowLevel}($p_i.task, p_i.cstr$) \\ 
}

\color{black}
Insert $n_0$ into $CT$\\
\While{$CT$ not empty}{
    $n \gets$ Lowest-cost node in $CT$\\
    $conflict \gets$ Find first conflict in $n$\\ 

    \If{$conflict = \emptyset$}{ \Return Success }

    \ForEach{$R_i$ in $conflict$}{
        \color{customcolor}
        \If{$R_i$ in $P$}{ $continue$ }

        \color{black}
        $n_c \gets n$\\
        $cstr \gets$ \textsc{ResolveConflict}($R_i$, $conflict$)\\    
        $n_c.cstr[R_i] \gets n_c.cstr[R_i] + cstr$\\
        $n_c.plan[R_i] \gets$ \textsc{LowLevel}($R_i.task, n_c.cstr[R_i])$\\
        Update $n_c.cost$\\
        Insert $n_c$ into $CT$\\
    }
}
\Return Failure
\end{algorithm}
} 

\subsection{Conflict-Based Search}

Conflict-Based Search (CBS) \shortcite{cbs} is a centralized constraint-based search algorithm designed to optimally solve the MAPF problem by using a two-tiered search approach consisting of a high-level search and a low-level search. The low-level search employs a pathfinding algorithm such as A* to determine the optimal sequence of actions for an agent to move from its starting point to its goal, adhering to any constraints imposed by the high-level search. These constraints specify which states must be avoided by agents during their path. PRISM utilizes a modified version of CBS, with pseudocode provided in Algorithm \ref{alg:cbs}. Our added and modified lines are highlighted in blue in Algorithm \ref{alg:cbs}. All other lines remain unchanged from the original CBS algorithm. For more details, please refer to \shortcite{cbs}. 

The high-level search manages paths for all agents collectively by identifying conflicts between agent pairs and introducing constraints to resolve them. This is achieved through a binary tree structure known as the conflict tree (CT), where each node represents a set of constraints applied to the agents. Each CT node, denoted $n$, includes a plan ($n.plan$), cost ($n.cost$), and the set of constraints ($n.cstr$) agents must follow. The goal of the high-level search is to navigate this tree to find the lowest-cost, conflict-free node. 

CBS begins by generating a root node $n_0$ in the CT, containing each agent's individually optimal path as computed by the low-level search in a constraint-free, decoupled manner. This initialization is modified in our pseudocode (lines 3-8) and the original CBS algorithm's root initialization is omitted from our pseudocode. The root node is inserted into the CT (line 9), and the search begins by removing the lowest-cost node from the priority queue (line 11). 

The first conflict $conflict = \langle R_i, R_j, v, t \rangle$ between agent pairs is identified at this node (line 12). To resolve the conflict, the tree splits to create two child nodes $n_i$ and $n_j$ (lines 15-23, excluding lines 16-17). These child nodes inherit constraints and plans from the parent node (line 18) and receive a new constraint (lines 19-20) to prevent the conflict from recurring. For example, constraints $\langle R_i, v, t \rangle$ and $\langle R_j, v, t \rangle$ are assigned to $n_i$ and $n_j$, respectively, creating two new search instances where neither agent can occupy vertex $v$ at time $t$. After adding constraints, the low-level search replans the paths for the affected agents to comply with the updated constraints (line 21), revising each plan's feasibility and cost (line 22). The new child nodes are then inserted into the CT (line 23), and CBS continues exploring until it finds a conflict-free node (lines 13-14) or until the CT is empty (line 10), indicating that no solution exists. By individually constraining each agent involved in a conflict, CBS systematically explores all possible solutions, ultimately reaching an optimal conflict-free node in the CT.

\section{Method} \label{Method}

\begin{figure*}[t]
    \centering
    \includegraphics[width=\linewidth]{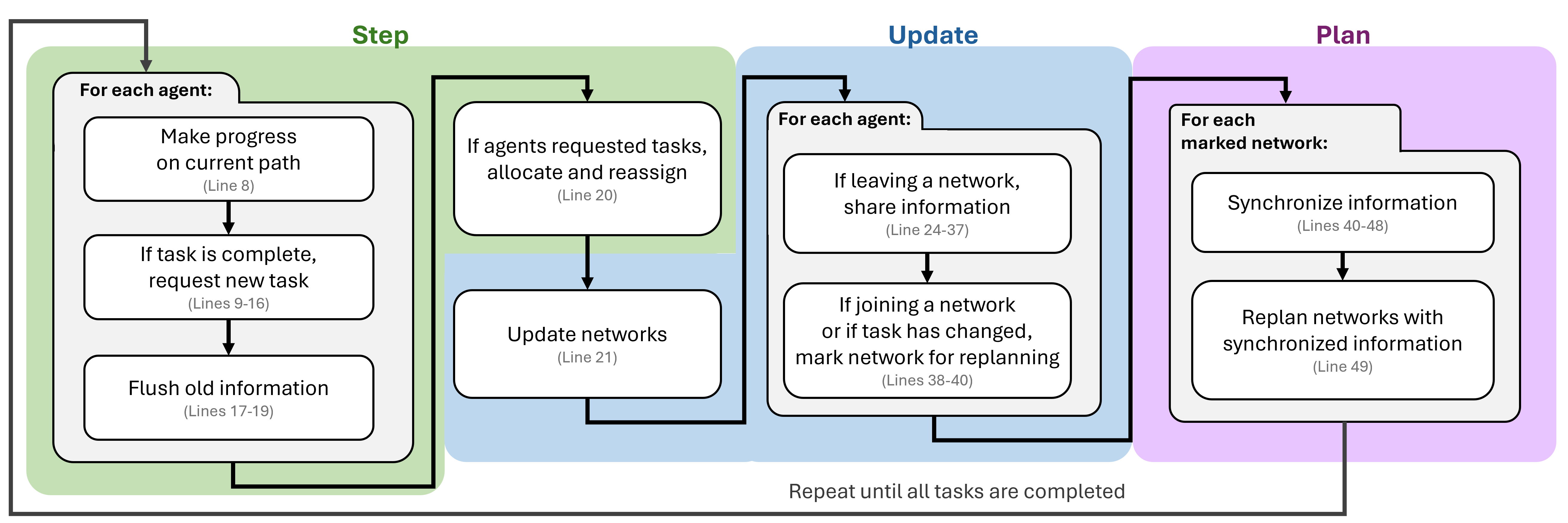}
    \caption{This flowchart illustrates the three phases of PRISM and provides a high level outline. Included are line numbers that correspond to specific steps in Algorithm \ref{alg:prism}.}
    \label{fig:flowchart}
\end{figure*}

Pathfinding with Rapid Information Sharing using Motion constraints (PRISM) focuses on enabling conflict-free pathfinding among agents while integrating dynamic task allocation. PRISM assumes the presence of a task allocator that assigns valid tasks to agents and can dynamically reassign unstarted tasks using real-time information from PRISM’s path calculations. Once an agent begins a mission task, it is locked to that agent and cannot be reassigned. However, tasks that have not yet been started can be reassigned as needed, allowing for flexible response to changing conditions.

PRISM utilizes a three-phase planning scheme combined with a Modified CBS for local network path planning, enabling seamless navigation and the integration of pathfinding with multi-task management. This approach supports advanced scheduling and planning, allowing agents to efficiently coordinate tasks and paths. In this section, we detail the information available to agents, describe the structure of the info packets used for communication, and provide an overview of the complete planning scheme.

\subsection{Agents} In PRISM, each agent $R_i$ manages a critical set of data that includes its unique identifier ($R_i.id$),  current assigned task ($R_i.task$), the planned path ($R_i.path$), applied motion constraints ($R_i.cstr$), and its collection of info packets ($R_i.pkts$). The task and path information directs the agent towards its intended destination and outlines the journey. Motion constraints resolve inter-agent conflicts and maintain continuity during replanning by ensuring that modifications made to paths in earlier iterations are preserved. These constraints are retained and referenced in subsequent planning phases, preventing inconsistencies and ensuring smooth trajectory adjustments. Additionally, agents store info packets they receive, which provide a snapshot of their current understanding of other agents' positions and trajectories.

        \begin{algorithm}[th!]
\scriptsize
\caption{PRISM}
\label{alg:prism}
\KwIn{Robots $R$, Tasks $T$}

\textsc{InitializePlans}($T$, $R$) \\
$t_{current} \gets 0$\\
$T_{started}$, $T_{done} \gets \emptyset$ \\
\While{$|T_{done}| \neq |T|$ \upshape{and} $\exists \  R_i \in R$ \upshape{not at rest}}{
    $A_{requested} \gets \emptyset$ \tcp*[f]{Step}\\
    \ForEach{$R_i$ \upshape{in} $R$}{ 
        $R_i.N_{previous} \gets R_i.N_{current}$ \\
        \textsc{Step}($R_i$) \\
        \If{$R_i.task$ \upshape{completed}}{
            \uIf{$R_i.task$ \upshape{is a mission task}} {
                $A_{requested} \gets A_{requested} \cup R_i$\\
                $T_{done} \gets T_{done} \cup R_i.task$ \\
            } 
            \ElseIf{$R_i.task$ \upshape {is a transition task}}{
                Update $R_i.task$ to its mission task \\
                $T_{started} \gets T_{started} \cup R_i.task$ \\
            }
            $R_i.packets \gets \emptyset$ \\
        }
        \ForEach{$p$ \upshape{in} $R_i.packets$}{
            \If{$p.t_{flush} \leq t_{current}$}{$R_i.packets.$\textsc{Flush}($p$) }
        }
    }
    \textsc{AllocateTasks}($T \setminus T_{started}$, $A_{requested}$) \\
    $Networks \gets$ \textsc{UpdateNetworks}($R$) \tcp*[f]{Update}\\
    $N_{replan} \gets \emptyset$ \\    
    \ForEach{$R_i$ \upshape{in} $R$}{
        \If{$R_i.N_{current} \neq R_i.N_{previous}$}{
            \ForEach{$R_j$ \upshape{in} $R_i.N_{previous}$}{
                \uIf{$R_i$ \upshape{or} $R_j$ \upshape{at rest}}{$t_{flush} \gets \infty$}
                \Else{$t_{flush} \gets $\textsc{CalculateFlushTime}($R_i.cstr$, $R_j.cstr$)}
                \If{$t_{flush} = \infty$} {
                    \If{$R_i$ \upshape{not at rest}} {
                        $R_i.packets \gets R_i.packets\  \cup \ $ \textsc{CreatePacket}($R_j.id$, $R_j.cstr, t_{current}, t_{flush}$)\\
                    } 
                    \ElseIf{$R_j$ \upshape{not at rest}} {
                    $R_j.packets \gets R_j.packets \  \cup \ $ \textsc{CreatePacket}($R_i.id$, $R_i.cstr, t_{current}, t_{flush}$) \\
                    }
                }
                \ElseIf{$t_{flush} > t_{current}$}{
                    $R_i.packets \gets R_i.packets\  \cup \ $ \textsc{CreatePacket}($R_j.id$, $R_j.cstr, t_{current}, t_{flush}$)\\
                    $R_j.packets \gets R_j.packets \  \cup \ $ \textsc{CreatePacket}($R_i.id$, $R_i.cstr, t_{current}, t_{flush}$)\\
                }
            }
            $N_{replan} \gets N_{replan} \cup R_i.N_{current}$
        }
        \If{$R_i.task$ \upshape{has changed}}{
            $N_{replan} \gets N_{replan} \cup R_i.N_{current}$
        }
    }
    \ForEach(\tcp*[f]{Plan}){$N_i$ \upshape{in} $N_{replan}$}{
        $Packets_{sync} \gets \emptyset$ \\
        \ForEach{$p$ \upshape{in} $N_i.packets$}{
            \If{$p.id$ \upshape{not in} $N_i$}{
                \uIf{$p.id$ \upshape{in} $Packets_{sync}$ \upshape{and} $p.t_{flush} \neq \infty$}{
                    \If{$Packets_{sync}[p.id].t_{receive} \geq p.t_{receive}$}{ continue }
                }
                $Packets_{sync}[p.id] \gets p$
            }
        }
        \textsc{Modified-CBS}($N_i$, $Packets_{sync}$)
    }
    $t_{current} \gets t_{current} + 1$
}
\end{algorithm}

\subsection{Info Packets} Info packets in PRISM facilitate effective replanning when agents move out of communication range from one another. Info packets have limited lifetimes and are used to increase global awareness during local in-network planning. Each info packet $p_j$ contains essential data about an agent: a unique identifier ($p_j.id$), its current task ($p_j.task$), motion constraints ($p_j.cstr$), a received time ($p_j.t_{receive}$), and a flush time ($p_j.t_{flush}$).

The agent's unique identifier in the info packet is used for precise identification, which is critical when multiple packets describing the same agent exist within the network. During replanning, the identifier helps networks determine which info packet to use, avoid redundancy, and ensure data consistency. Additionally, if an agent is holding an info packet and comes within direct communication range of the agent described by the packet, it uses the identifier to discard the packet to prevent unnecessary duplication. 

The info packet's task and motion constraints are used to represent an agent's path. Instead of storing the comprehensive details of an agent's path, such as every configuration or position along the path, an agent's task and motion constraints can be used as an alternative representation. These elements are sufficient to reconstruct the agent's path as needed using a low-level pathfinder, keeping the info packets compact and efficient. 

The received time of an info packet is crucial during synchronization in replanning phases, enabling agents to identify and utilize the most recent packet available for any given agent. This ensures that plans are always based on the latest information. 

The received time of an info packet is crucial during synchronization in replanning phases, allowing agents to identify and utilize the most recently received packet for any given agent. This ensures that, when multiple packets describe the same agent, the planned paths are consistently based on the latest available information.

The flush time determines a packet's relevance to its holder. Once the flush time expires, the packet is deemed outdated and discarded, maintaining data freshness across the network.  This mechanism ensures that agents retain only relevant information packets, preventing excessive accumulation and reducing the risk of outdated or incorrect information propagating widely. 

The flush time is calculated based on the last motion constraint applied between the info packet's subject agent and the holder, reflecting the estimated time of potential interference between the agents. As a modification to standard CBS constraints, we now track the source of each constraint. For example, if a constraint is applied to agent $R_i$ at time $t$ due to a conflict with agent $R_j$, we record this constraint along with the information that it originated from $R_j$. If $R_i$ exits the shared network with $R_j$, we determine the last relevant constraint by finding the last constraint $R_j$ applied to $R_i$ and vice versa. The latest time among the constraints across both agents is then used as the flush time, providing an estimate for when $R_i$ and $R_j$’s paths diverge and thus when the info packet will no longer be relevant to its holder.

To create an info packet describing agent $R_i$, we copy the agent's unique identifier ($R_i.id$), task ($R_i.task$), and motion constraints ($R_i.cstr$), record the received time, and calculate the flush time. While the identifier, tasks, and motion constraints within each info packet are unique to the agent it describes, the received and flush times are shared between agent pairs.

\subsection{Overview} 
PRISM operates through a structured three-stage process: step, update, and plan, with a detailed flowchart depicted in Figure \ref{fig:flowchart} and corresponding algorithmic steps outlined in Algorithm \ref{alg:prism}. This structured approach ensures efficient and conflict-free navigation with a multi-agent system. In this subsection, we reference lines from both Algorithm \ref{alg:cbs} and Algorithm \ref{alg:prism}. Subsections \ref{sub:step} to \ref{sub:plan} have line references pertaining to Algorithm \ref{alg:prism}. Subsection \ref{sub:cbs} has line references pertaining to Algorithm \ref{alg:cbs}.

The process begins with each agent independently planning its initial, constraint-free path based solely on its assigned task (Alg.  \ref{alg:prism}, line 1), without considering other agents’ positions or plans. We initialize a clock (line 2) which is incremented with each iteration of the the three-stage process (line 39) and initialize the set of started and completed tasks (line 3). Then, we begin and repeat the three-stage process until all tasks have been completed and all agents are at rest (line 4). 

\subsubsection{Step Phase} \label{sub:step}
In the step phase (Alg. \ref{alg:prism}, lines 5-19), agents cache their current local network as their previous network for use in the update phase (line 7) and then advance one timestep along their planned paths (line 8). If an agent completes its task during this step (line 9), it is determined whether the task is a mission task (line 10) or a transition task (line 13). For mission tasks, the agent is marked as having requested a new task, and the completed mission task is marked as finished (lines 10-12). For transition tasks, the agent updates its task to the corresponding mission task, marking it as started (lines 13-15). Marking a mission task as started ensures that the task allocator does not redundantly reassign it to another agent. Upon completing a task, the agent flushes all associated info packets (line 16), as these packets pertain specifically to the current task path and are no longer relevant when transitioning to a new task. If an agent has not completed its task, it checks for outdated info packets and discards them based on their flush time (lines 17-19).

After all agents have progressed along their paths, agents requesting new tasks are identified, and any unstarted tasks are allocated to them (line 20). This allocation step may involve reassigning tasks among agents. If an agent's original task is reassigned, the agent will either transition to its newly assigned task or, if no new task is assigned, proceed to the goal position of its previous task. Whenever an agent's task changes, all associated info packets are flushed; this step is omitted from the pseudocode for brevity.

\begin{figure*}[ht]
    \centering
    \includegraphics[width=\linewidth]{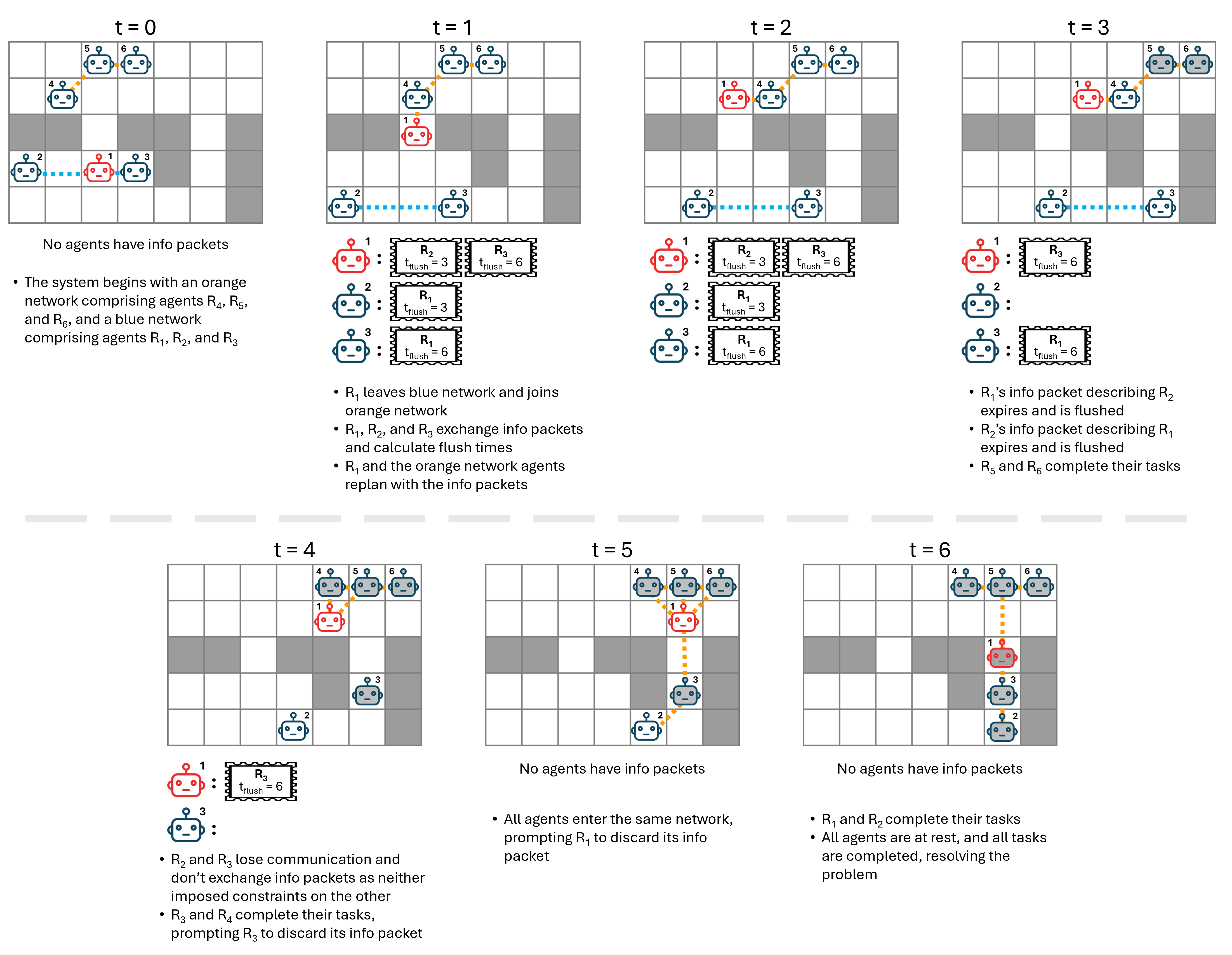}
    \caption{Example of a system consisting of two local networks using bounded info packets. Greyed-out agents indicate resting agents. Bounded info packets are discarded once their flush time expires, the holder re-establishes communication with the packet’s origin agent, or the holder completes its task.}
    \label{fig:bounded}
\end{figure*}

\subsubsection{Update Phase} \label{sub:update}
In the update phase (Alg. \ref{alg:prism}, lines 21-40), agents adjust their communication networks based on their updated positions (line 21), using multi-hop communication to establish connections in accordance with predefined protocols. Specific actions are triggered when an agent exits its previous network and joins a new network (line 24), or changes tasks (line 39). When an agent's local 
 network changes (line 24), agents broadcast and exchange info packets with the remaining network members (line 25). If either agent is at rest, we create an info packet with an infinite flush time (lines 26-27); otherwise, a bounded flush time is calculated based on their respective constraints (lines 28-29). The rationale for selecting either an infinite or bounded flush time is discussed in the following paragraph. If the flush time is infinite (line 30), the info packet is only created and applied to the moving agent (lines 31-34). If the calculated flush time exceeds the current time (line 35), the agents remain relevant to each other for a set duration, and info packets are created for both agents, with the current time recorded as the received time and the appropriate flush times applied (lines 36-37). The newly joined local network is then marked for replanning to ensure the agent's information is incorporated into the teams plan (line 38). When an agent changes tasks, the network is also flagged for replanning (lines 39-40) to incorporate the updated task into the network's strategy. Figure \ref{fig:bounded} illustrates an example of networks utilizing bounded info packets, while Figure \ref{fig:infinite} depicts networks with infinite info packets.

\begin{figure}[ht]
    \centering
    \includegraphics[width=\linewidth]{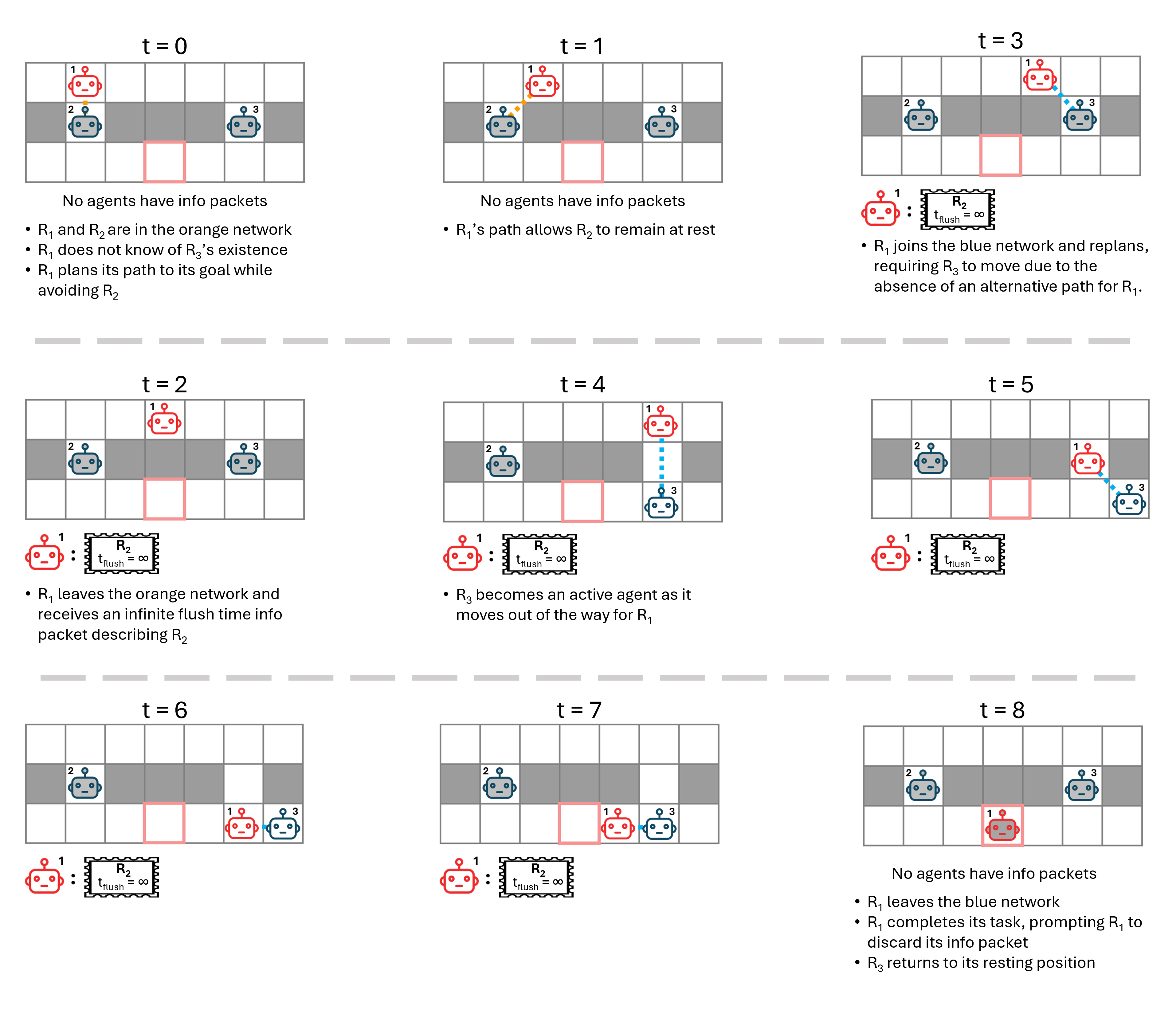}
    \caption{Example of a system with three agents using infinite info packets. Greyed-out agents ($R_2$ and $R_3$) indicate those that have reached their resting positions, while the red square marks the goal position of the red agent ($R_1$). This perspective is from agent $R_1$; at timesteps 0-2, $R_1$ is unaware of $R_3$'s existence (indicated by $R_3$ being whited out). Upon receiving an infinite info packet, $R_1$ remembers $R_2$'s presence even upon leaving the network.}
    \label{fig:infinite}
\end{figure}

Info packets with infinite flush times are critical for preventing deadlock scenarios involving resting agents and are only applied to moving agents, while info packets with bounded flush times are used for moving agents. When both agents are in motion, even as info packets expire, the dynamic interactions within the team naturally resolve deadlock scenarios over time, which will be proven in Section \ref{section:proof}. However, when some agents are at rest and at least one is moving, a specific deadlock risk arises. If a moving agent is blocked by resting agents and can communicate with only one of them at a time, an expired info packet may fail to provide sufficient information for the moving agent to navigate to its goal effectively.

For example, as shown in Figure \ref{fig:infinite}, let us assume we have a moving agent $R^m_1$ and two resting agents $\{R^r_2, R^r_3\}$. If  $R^m_1$ encounters the $R^r_2$, it might choose to avoid it rather than requesting $R^r_2$ to move. $R^m_1$ will retain an info packet for a limited time, but if the packet expires before $R^m_1$ encounters $R^r_3$, it may similarly avoid $R^r_3$ instead of prompting movement. This could result in $R^m_1$ thrashing indefinitely between $R^r_2$ and $R^r_3$, unable to progress toward its goal. To prevent such deadlocks, info packets containing the positions of resting agents are assigned infinite flush times, ensuring their information remains available to $R^m_1$ for the duration of its movement. 

It is important to note that infinite flush times are applied only to the moving agent in this scenario. This ensures that resting agents do not accumulate info packets that they cannot flush. If a resting agent is required to move from its position to enable $R^m_1$ to reach its goal, it is no longer considered at rest and instead receives a bounded info packet as necessary. Infinite info packets are applied solely to the agent holding them and are not shared with others, ensuring that the contents of these info packets remain local, avoiding unnecessary propagation though the environment. By containing the scope of these packets, deadlock situations can be resolved effectively with minimal impact on other agents, as we will demonstrate in Section \ref{section:proof}.  

Lastly, although info packets with infinite flush times persist longer than those with bounded times, their lifetime is inherently limited to the duration of the current task. When a moving agent completes its task and is assigned a new one, it automatically flushes all associated info packets, including those with infinite flush times. This mechanism prevents the unchecked accumulation of outdated packets and ensures that info packets remain task-specific.

\subsubsection{Plan Phase} \label{sub:plan}
In the plan phase (Alg. \ref{alg:prism}, lines 41-50), local networks marked for replanning undergo synchronization and path adjustment. Each network aggregates its info packets and synchronizes them to eliminate redundancies (lines 42-48), using each packet's unique identifier and received time (lines 44-46) to compile a set of non-redundant packets containing the most up-to-date information (line 48).  During synchronization, packets describing agents that are also part of the current network are excluded (line 44). Similarly, packets with infinite flush times (line 45) and older packets (line 46) are omitted from the synchronized set. However, packets with infinite flush times are still applied to their respective holders during the network's replanning phase, but they do not influence the paths of other network members. The resulting synchronized packets are then used in the Modified CBS to replan the paths of all agents in the local network (line 47).

\begin{figure}[t]
    \centering
    \includegraphics[width=\linewidth]{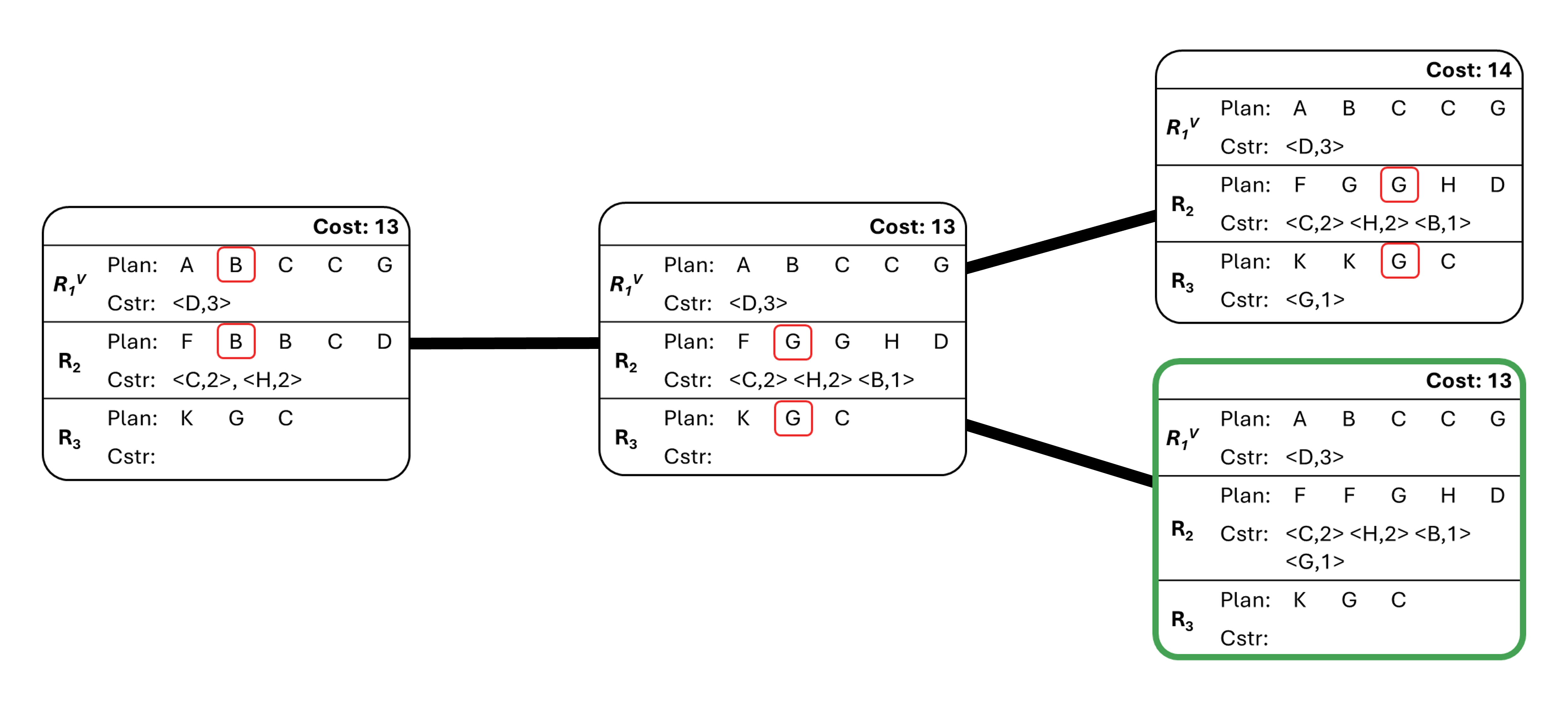}
    \caption{Example of a conflict tree from a Modified CBS call involving three agents. $R_1^v$ is an info packet representing agent $R_1$, acting as a virtual agent within the conflict tree, while $R_2$ and $R_3$ are network agents. Conflicts between $R_1^v$ and network agents result in the creation of a single child node, preserving the static path of the virtual agent. In contrast, conflicts between network agents follow standard CBS behavior, generating two child nodes. Conflicts are highlighted in red and the solution node is highlighted in green.}
    \label{fig:mod_cbs}
\end{figure}

\subsubsection{Modified CBS} \label{sub:cbs}
In the Modified CBS framework, we use agents within a local network and their associated info packets to determine non-conflicting paths for all agents in the network. First, we create a root node, populating it with each agent's information (Alg.  \ref{alg:cbs}, lines 3-5) and associated info packets (lines 6-8). This includes copying the constraints and path for each agent. For info packets, each packet’s task and motion constraints are used to generate the path of the packet’s agent, which is expanded using a low-level search algorithm (line 8).

CBS operates with a key modification: info packets are treated as virtual agents. In the event of a conflict between an info packet agent and a network agent, the info packet imposes constraints on the network agent, but network agents cannot impose constraints on info packet agents (lines 16-17). This rule ensures that the paths of info packet agents remain static, preserving the integrity of the information they carry, while allowing only the network agents' paths to be modified. Consequently, conflicts between info packet agents and network agents result in the creation of a single conflict tree node. An example of this behavior is illustrated in Figure \ref{fig:mod_cbs}.

After initializing the root node and applying the info packet constraint rule, CBS largely proceeds as usual, with a minor adjustment to the low-level search. During low-level search calls (line 21), infinite info packets are treated as permanent obstacles for the packet holder. The resting agent's position, stored in the info packet, is considered a static obstacle exclusively for the agent holding the packet. Once a local network team solution is found, the updated paths and constraints are communicated back to each network agent.

This iterative, three-stage planning cycle continues until all tasks are completed and all agents are at rest, enabling PRISM to dynamically adapt to changing task specifications and agent conditions. When an agent’s task changes mid-execution, info packets disseminate the adjustment to relevant agents throughout the environment. The system’s regular flushing mechanism ensures the accuracy and relevance of shared information by removing outdated or incorrect data. A key contribution of this work is the strategic use of info packets to preserve and share information among agents, integrating them as virtual agents within CBS to maintain high-quality paths. This dynamic and responsive approach effectively addresses the complexities of real-time multi-agent coordination in evolving environments.

\section{Theoretical Analysis} \label{section:proof}
In this section, we will prove that PRISM is complete and capable of resolving all solvable deadlock situations. To construct this proof, we begin by stating our assumptions and defining key terms and concepts, such as deadlock and resources. Then, we provide a high-level overview of our proof before delving into the necessary lemmas to prove that both Modified CBS and PRISM are complete. 

\subsection {Assumptions \& Definitions}

\begin{definition}
    An agent is considered to be \textbf{at rest} if it has completed its assigned tasks, has no other task to complete, and requires no further motion at the current timestep. Its resting position is a position that is neither a current start or a goal position for another agent. If a resting agent is prompted to move temporarily away from its resting position, it is no longer considered to be at rest. 
\end{definition} \label{def:at_rest}

\begin{definition}
    An \textbf{info packet} in PRISM is a structured message exchanged between agents to preserve and transmit critical planning information after communication is lost. Each info packet captures a snapshot of an agent’s identity, task, and motion constraints at the time of network separation, along with a received time indicating when the packet was generated and a flush time denoting its relevance duration. This flush time can be bounded or infinite. Bounded flush times are computed based on the last known conflict between the agents and reflects the expected time window in which the agents’ paths may still intersect. During local replanning, info packets enable disconnected agents to make informed decisions by leveraging recent state information about their peers, ensuring coordination even outside direct communication. The identifier ensures consistency and prevents redundancy, while the minimal representation of motion constraints allows efficient path reconstruction without retaining full trajectories.
\end{definition}

\begin{definition}
    A \textbf{local network} is a group of agents that belong to the same connected component and can communicate with one another using multi-hop communication. Each local network maintains state information about its constituent agents and supports coordinated planning through a centralized subproblem: a designated agent solves the MAPF instance using the collective information, then distributes the resulting plans to all members. This ensures that all agents in the network operate with a consistent view of the current state and planned actions.
\end{definition}

\begin{definition}
    The \textbf{system} refers to the collection of all local networks operating together in a shared environment. 
\end{definition}

\begin{definition}
    A \textbf{resource} is a specific element of the environment's representation, such as a single vertex (e.g., a graph node or grid cell) or a single edge, that an agent can occupy or traverse. Together, all resources make up the free space in the environment's representation.
\end{definition}

\begin{definition}
    A \textbf{deadlock} is a situation in which no agent can proceed because all involved agents are waiting on others to release shared resources. 
\end{definition}

\begin{assumption}
    The multi-task multi-agent pathfinding (MAPF) problem addressed by PRISM assumes that each agent is assigned a solvable motion task with a unique start and goal position, distinct from those of all other agents. This guarantees that no two agents share the same start or goal, ensuring that all tasks are inherently solvable and preventing deadlocks caused by overlapping task assignments. 
\end{assumption}

\begin{assumption}
    Once agents complete their tasks and enter a resting state, they remain active participants within the system. Resting agents continue to communicate with other agents, exchange info packets, and, if necessary, move to facilitate the progress of active agents. When a resting agent relocates to assist others, it returns to its original resting position after sufficiently clearing the way.
\end{assumption}

\begin{assumption}
    We consider a constant environment graph $G = (V, E)$, which is shared and fully known by all agents during planning. As defined in the problem statement, at each timestep $t$, an agent located at vertex $v \in V$ may either transition to a neighboring vertex $v' \in V$ via an edge $(v, v') \in E$, or remain at $v$. Each transition and wait action incurs a finite cost, and agents cannot execute infinite cycles within the graph, ensuring all transitions progress toward a goal or a stationary state. 
\end{assumption}

\begin{assumption}
   Agents operate under a perfect communication model: all transmitted info packets are reliably received by intended recipients within the same communication network without loss, delay, or corruption. Additionally, even with constrained communication, agents are able to exchange information when they approach a potential collision. The system assumes sufficient proximity and time for communication and replanning to occur before a collision becomes unavoidable.
\end{assumption}

\subsection{Proof Sketch}

We first prove the completeness of Modified CBS (described in Section \ref{sub:cbs}) within individual networks, accounting for the inclusion of info packets. Next, we demonstrate how the integration of Modified CBS, info packets, and selective constraint application ensures deadlock-free operation in all solvable scenarios, establishing PRISM's completeness. 

Modified CBS extends CBS by incorporating info packets, relying on the original completeness guarantees of CBS. As a complete MAPF algorithm \cite{cbs}, CBS guarantees a solution if one exists and can identify unsolvable instances with extensions such as \cite{yu2015pebble}. 

PRISM addresses local subproblems using information from individual local networks. Within each network, agents make local decisions governed by network agents and info packets. We show that through Modified CBS, these local decisions still preserve CBS's completeness guarantees and are sufficient to ensure deadlock resolution, confirming PRISM's overall completeness.

\subsection{Completeness of PRISM}

\begin{property}
    When operating within a local network without the use of info packets, Modified CBS reduces to standard CBS. As shown in the original completeness proof of CBS \cite{cbs}, the algorithm is guaranteed to find a valid solution if one exists. Therefore, Modified CBS inherits this completeness property and will return a valid solution in such settings.
\end{property}
\begin{lemma}
    Modified CBS will return a valid solution within a local network consisting of info packets with bounded flush times.
\end{lemma}

\begin{proof}
    We assume a set of solvable motion tasks (Assumption 6.1), a known environment graph (Assumption 6.3), and perfect communication (Assumption 6.4). Let $\mathcal{N}$ denote a local communication network of agents executing Modified CBS with the exchange of bounded info packets. 

    Recall that standard CBS is known to be complete \cite{cbs}. The overall state space of agent configurations (including waiting actions) is infinite. There are also infinitely many valid solutions because an agent may wait arbitrarily long. Our focus is on the existence of at least one valid conflict resolution sequence. Standard CBS is complete in the sense that if a valid collision-free plan exists, there is at least one finite sequence of conflict resolutions (i.e., a finite branch of the conflict tree) that leads to such a solution.

    Under PRISM's policy, a bounded info packet is broadcast to every agent in the local network immediately after either (1) a new agent enters the network (Algorithm \ref{alg:prism}, Line 24)  or (2) any member's task changes (Algorithm \ref{alg:prism}, Lines 39). Each bounded info packet can introduce additional bounded motion constraints to the low-level planning process. These are applied according to the following constraint rule, which is also the only algorithmic modification that is introduced to the CBS framework: 
    \begin{enumerate}
        \item \textbf{Constraint Rule 1 (One-Way Constraint Influence for Bounded Info Packets):} A bounded info packet can impose a temporal and spatial motion constraint on every network agent $R_i \in \mathcal{N}$ that is in the same network $\mathcal{N}$ as the packet's holder. These constraints restrict the agents from occupying specific vertices at designated times, but they do not influence the planning of the agent whose behavior is described by the info packet (Algorithm \ref{alg:cbs}, Lines 16-17). 
    \end{enumerate}

    Given these constraint rules, any constraint introduced by a bounded info packet is adhered to by the affected agent in one of two ways: waiting or replanning. An agent can choose to execute wait actions to delay arriving at a vertex until after the constrained timestep has passed. More formally, if a constraint forbids an agent from occupying a vertex at time $t$, the agent can wait at its current vertex such that it does not reach the forbidden vertex until $t' > t$, thereby satisfying the constraint. Alternatively, if waiting is not feasible or efficient, the agent can replan its path to avoid the constrained region altogether. Standard CBS guarantees that if a valid solution exists, there is at least one finite sequence of conflict resolutions in which the constraints are resolved via adjusted paths. 

    Although the set of valid solutions is itself infinite because an agent can always wait arbitrarily long, the key observation is that the introduction of bounded info packet constraints does not eliminate every valid solution. In fact, if a valid collision-free plan exists, then by the completeness of standard CBS, there still remains a finite conflict resolution sequence that achieves that plan. The bounded constraints, as imposed by Constraint Rule 1, only restrict certain states temporarily. As a consequence, Modified CBS still only needs to resolve a finite number of conflicts before arriving at a valid solution.

    Thus, despite the additional constraints introduced by bounded info packets, the solution space still contains a valid solution that is guaranteed to be explored by Modified CBS, given that a valid solution exists. Because every bounded constraint carries a finite constraint application time $t$ the maximum delay introduced by waiting out a specific motion constraint is strictly bounded by $t - t_{current}$; hence an agent cannot be forced into an infinite wait loop, and eventual progress is guaranteed. Therefore, Modified CBS remains complete under the application of bounded info packet constraints. 
    
\end{proof} 


\begin{lemma}
    A local network utilizing infinite info packets will guarantee the discovery of a valid solution, if one exists, when using Modified CBS.
\end{lemma}

\begin{proof}
    We assume a set of solvable motion tasks (Assumption 6.1), active resting agents (Assumption 6.2), a known environment graph (Assumption 6.3), and perfect communication (Assumption 6.4). Let $\mathcal{N}$ be a local network of agents executing the Modified CBS algorithm. Within this network, infinite info packets may be exchanged when a moving agent exits a network in which it was constrained by a resting agent. These info packets are retained for the duration of the agent’s current task and introduce persistent constraints into its planning process.
    
    The additional constraints introduced by infinite info packets are governed by the following constraint rules: 
    \begin{enumerate}
        \item \textbf{Constraint Rule 2 (Holder-Specific Influence for Infinite Info Packets):} Constraints derived from infinite info packets apply exclusively to the agent that holds the packet. No other agent is affected, and the constraint does not propagate through the system or induce branching in other agents’ conflict trees. As a result, the only modification introduced by infinite info packets is to the low-level search of the packet holder; the high-level structure of CBS and the conflict tree remain unchanged (Algorithm \ref{alg:cbs}, Line 21). 
        \item \textbf{Constraint Rule 3 (Task-Specific Lifetime):} Irrespective of the flush time, all info packets are immediately discarded when an agent's task changes or is completed. This rule ensures that info packets do not persist into subsequent planning cycles (Algorithm \ref{alg:prism}, Lines 18-19). 
    \end{enumerate}

    Before an infinite info packet is created (Algorithm \ref{alg:prism}, Lines 31-34), the algorithm explicitly verifies that an alternative path exists for the holder. Consequently, the holder always follows a finite-length route and never waits indefinitely due to the infinite info packet. This verification guarantees that imposing the persistent constraint will not remove all valid solutions; rather, it guarantees that at least one valid path remains available. Within the overall infinite set of valid solutions, by verifying the existence of an alternate path before creating an infinite info packet, the algorithm ensures that the additional persistent constraint does not block this finite solution branch. Rather, it restricts the search to paths that respect the confirmed alternative routes, even after the holder loses communication with the conflicting resting agent.

    The number of infinite constraints that an agent can hold is finite and bounded by the number of agents in the system. Modified CBS further mitigates long-term accumulation of infinite info packets by eventually selecting the resting agent for replanning if avoidance becomes suboptimal or infeasible. As a result, any one agent will be constrained by at most $|R| - 2$ info packets before Modified CBS opts to move a resting agent. 

    Once the holder's task is completed, all associated info packets are removed. This ensures that previously restricted portions of the search space become available again for future planning tasks, preserving the long-term completeness of the system. 

    Thus, with the application of Constraint Rules 2 and 3, even though infinite info packet constraints persist for the duration of the current task and thereby reduce the scope of potential solutions, there always remains at least one finite sequence of conflict resolutions that yields a valid collision-free plan. Therefore, Modified CBS remains complete in the presence of infinite info packets.
\end{proof}


\begin{theorem}
    In a local network, Modified CBS maintains completeness and will return a valid collision-free plan if one exists.
\end{theorem}

\begin{proof}
   We assume a set of solvable motion tasks (Assumption 6.1), active resting agents (Assumption 6.2), a known environment graph (Assumption 6.3), and perfect communication (Assumption 6.4). Let $\mathcal{N}$ be a local network of agents executing the Modified CBS algorithm. Within $\mathcal{N}$, agents exchange info packets during the planning process. Two types of info packets are employed: 
   \begin{itemize}
       \item \textbf{Bounded Info Packets} add temporary constraints and are governed by Constraint Rule 1 (One-Way Constraint Influence for Bounded Info Packets) and Constraint Rule 3 (Task-Specific Lifetime). 
       \item \textbf{Infinite Info Packets} are generated when a moving agent exits a network in which it was previously constrained by a resting agent. These packets are governed by Constraint Rule 2 (Holder-Specific Influence for Infinite Info Packets) and subject to Constraint Rule 3 (Task-Specific Lifetime). 
   \end{itemize}

    By Lemma 6.1, the additional constraints from bounded info packets, although they may restrict certain states in the planning process for the network of agents, do not eliminate all valid solutions (Constraint Rule 1). In every instance, any constraint can be adhered to by either executing a wait action or by replanning an alternate route. Thus, there always remains at least one finite branch (i.e., a finite sequence of conflict resolutions) leading to a valid solution. 

    Similarly, by Lemma 6.2, before an infinite info packet is created the algorithm verifies that the packet's holder has an alternative path that avoids the conflicting resting agent. Thus, when the persistent constraint is applied (governed by Constraint Rule 2), it does not remove all valid solutions. Instead, it restricts the planning search to paths that respect the confirmed alternate routes while leaving the finite branch corresponding to a valid collision-free plan intact. 

    Since the set of valid solutions is infinite, completeness is understood in the sense that there exists at least one finite branch of the conflict tree that leads to a valid solution. As both bounded and infinite info packet constraints preserve the existence of such a finite branch via their respective constraint rules and the checks for alternate paths, Modified CBS is guaranteed to eventually explore it if a valid collision-free plan exists.

    All three lemmas invoke Algorithm \ref{alg:prism} (Lines 24-40) for packet creation and Algorithm \ref{alg:cbs} (Lines 16-21) for one-way constraint application, making the completeness argument explicitly dependent on the stated information-sharing policy and agent responses. Because Modified CBS is invoked each time the network gains a member or any member’s task changes (Algorithm \ref{alg:prism}, Lines 41-49), the constraints and paths supplied to the solver are finite and fully reflect the network’s latest state, capturing every conflict currently known and allowing them to be resolved in that invocation. Thus, combining these results, Modified CBS maintains completeness within any local network $\mathcal{N}$: if a valid solution exists, a finite conflict resolution sequence will be found and the corresponding collision-free plan returned. 
    
\end{proof}

\begin{implication}
    Any potential deadlock among agents in a local network is guaranteed to be resolvable through Modified CBS.
\end{implication}


\begin{lemma}
    Modified CBS and info packets resolve any potential deadlock in systems with agents belonging to different local networks.
\end{lemma}

\begin{proof}
    Let Assumptions 6.1-6.4 hold. Let $R_1$ and $R_2$ be two agents in distinct local networks $\mathcal{N}_1$ and $\mathcal{N}_2$, respectively, such that a deadlock exists between these agents. Although multi-agent deadlocks may involve more than a pair of agents, the structure of the Modified CBS algorithm, which resolves conflicts via pairwise constraint application, ensures that resolving a pairwise conflict is sufficient to extend the resolution to the entire deadlock. 

    Three mechanisms provide a layered strategy for resolving inter-network deadlocks: 
    \begin{enumerate}
        \item \textbf{Resolution via Info Packets: } An agent $R_3$ that was previously part of $\mathcal{N}_1$ may later join $\mathcal{N}_2$, carrying with it a bounded info packet that describes $R_1$'s current motion constraints and task information. Once the info packet is transferred into $\mathcal{N}_2$ , Constraint Rule 1 enables $R_2$ to be influenced by the contents of the info packet such that it can replan its path to avoid conflict with $R_1$. In this way, the pairwise deadlock between $R_1$ and $R_2$ can be resolved. 

       This mechanism requires that the packet be recent enough to reflect $R_1$'s current or still-relevant motion constraints. If the info packet is outdated (e.g., if $R_1$'s path has since changed), then the packet may no longer be sufficient to guide $R_2$ toward resolution. In such cases, resolution must occur via one of the other two mechanisms. 
        
        \item \textbf{Resolution via Internal Network Influence: } Agents can directly influence each other through interactions with others in their respective networks. For instance, an agent in $\mathcal{N}_1$ may influence $R_1$'s path in a way that resolves the deadlock with $R_2$. While not sufficient on its own to guarantee resolution of all deadlocks, this mechanism can help eliminate partial dependencies and eventually lead to resolution through info packets or network merging. 
        
        \item \textbf{Resolution via Network Merging: } If info packet-based resolution is insufficient, then agents will continue to progress along their planned paths. By Assumption 6.4, we assume sufficient proximity and time for communication and replanning to occur before a collision becomes unavoidable. Thus, eventually, if a deadlock has not been resolved, the conflicting agents will eventually merge into a single local network. Once merged, complete state and constraint information is shared among all network agents, and by Theorem 6.1, Modified CBS will resolve any remaining conflicts within this unified network. 
    \end{enumerate}

    Under PRISM’s information‑sharing policy, an infinite info packet is created exactly when a moving agent exits a local network that still contains the resting agent constraining it (Algorithm \ref{alg:prism}, Lines 25-26). The packet therefore records the resting agent’s position at that instant, ensuring the infinite motion constraint mirrors the state at the precise moment of network‑membership change.

    A moving agent may ``thrash" between two local networks when each network’s bounded info packet guides it toward the other, but expires before the agent reaches its target. Consider an agent oscillating between Networks A and B. Upon leaving Network A, the agent receives a bounded info packet directing it toward Network B. However, if this packet expires before the agent arrives, it cannot influence the subsequent planning cycle. As a result, upon entering Network B, the agent may be redirected back to Network A, potentially initiating a repeated cycle.
    
    However, while the agent thrashes, other robots within Networks A and B continue progressing toward their tasks, thereby altering the state of the local networks. These dynamic changes typically disrupt the precise conditions required for persistent thrashing. In the rare event that thrashing does persist, it ultimately resolves when other agents complete their tasks and transition into resting states. Resting agents issue infinite info packets upon the thrashing agent’s exit from their networks, imposing permanent constraints that guide it along pre-verified alternative paths. These infinite packets serve as enduring guidance for all future planning iterations, ensuring the agent no longer re-enters the thrashing cycle. Therefore, the transient nature of bounded packets, coupled with the ongoing progress of other agents and the eventual issuance of infinite info packets, guarantees the termination of thrashing after a finite number of cycles.
    
    Together, these mechanisms form a layered resolution strategy. If info packets or internal influence are insufficient, network merging ensures that any remaining deadlock will eventually be resolved. Therefore, Modified CBS, augmented with info packets and executed within the dynamic network structure of PRISM, guarantees completeness even across disjoint local networks. 
\end{proof}


\begin{theorem}
PRISM is complete and can resolve all solvable deadlock situations.      
\end{theorem}

\begin{proof}
    Let Assumptions 6.1–6.4 hold. To establish completeness, we must show that for any system configuration, if a solution exists, PRISM will eventually find it; that is, deadlocks can always be resolved both within local networks and across different networks.

    Deadlock resolution within a single local network is guaranteed by Theorem 6.1, which proves the completeness of Modified CBS when applied locally. Lemma 6.3 extends this result to systems where deadlocks involve agents from different local networks. This lemma demonstrates that inter-network deadlocks are resolved through mechanisms such as indirect planning information transfer, internal network influence, and eventual network merging. Once conflicting agents merge into a single network, the local completeness result (Theorem 6.1) applies.

    By combining the local resolution of deadlocks (Theorem 6.1 with Lemmas 6.1 and 6.2) with the resolution of inter-network deadlocks (Lemma 6.3), we conclude that in the infinite set of valid solutions, PRISM ensures the existence of at least one finite branch of conflict resolutions leading to a valid collision-free plan. Therefore, PRISM is a complete algorithm that will eventually find a valid plan if one exists.
\end{proof}

\section{Experiments and Results} \label{Experiments}


\begin{figure*}[t!]
    \centering
    \begin{subfigure}[t]{0.4\textwidth}
        \centering
        \includegraphics[width=\textwidth]{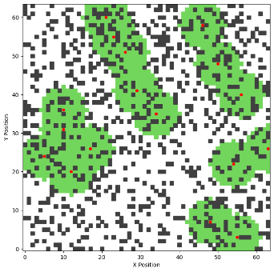}
        \caption{random-32-32-20 w/ 10\% Prox}
        \label{fig:random-prox}
    \end{subfigure}%
    ~ 
    \begin{subfigure}[t]{0.4\textwidth}
        \centering
        \includegraphics[width = \textwidth]{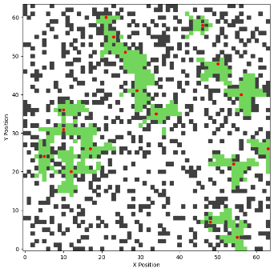}
        \caption{random-32-32-20 w/ 10\% LoS}
        \label{fig:random-los}
    \end{subfigure}
    ~ 
    \begin{subfigure}[t]{0.25\textwidth}
        \centering
        \includegraphics[width = \textwidth]{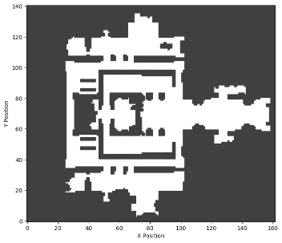}
        \caption{ht\_chantry}
        \label{fig:ht-chantry}
    \end{subfigure}
        ~ 
    \begin{subfigure}[t]{0.25\textwidth}
        \centering
        \includegraphics[width = \textwidth]{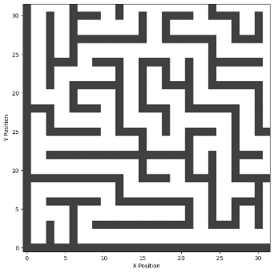}
        \caption{maze-32-32-2}
        \label{fig:maze-32}
    \end{subfigure}
        ~ 
    \begin{subfigure}[t]{0.25\textwidth}
        \centering
        \includegraphics[width = \textwidth]{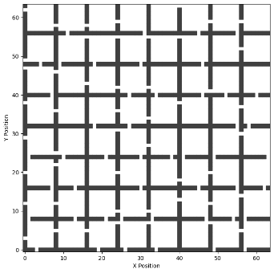}
        \caption{room-64-64-8}
        \label{fig:room-64}
    \end{subfigure}
        ~ 
    \begin{subfigure}[t]{0.45\textwidth}
        \centering
        \includegraphics[width = \textwidth]{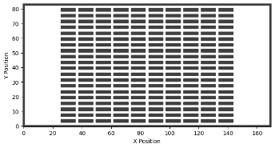}
        \caption{warehouse-10-20}
        \label{fig:warehouse-10}
    \end{subfigure}
    \caption{This figure shows the environments and examples of (a) 10\% proximity and (b) LoS communication protocols. Note that for LoS, it is assumed agents can see all other agents within a 4 diameter grid cell proximity of itself to avoid immediate collisions. Agent positions are shown in red and the range of communication is shown in green.}
    \label{fig:exp_setup}
\end{figure*}
In this section, we will discuss how we tested the soundness, robustness, and scalability of PRISM through two sets of experiments. We begin with a description of the experimental setup, followed by an analysis of PRISM’s performance in these experiments. Finally, we conclude with a discussion of the results from our physical system validation.

Empirical evaluations demonstrate that PRISM achieves exceptional scalability, robustness, and efficiency across diverse scenarios. Compared to centralized CBS and decentralized Token Passing with Task Swaps (TPTS), PRISM supports significantly more agents and tasks, maintains high solution quality, and delivers faster computation times, even under constrained communication conditions. Its performance is particularly strong in narrow passage environments and low-connectivity networks, where it outperforms CBS and TPTS. Additionally, PRISM’s adaptability to varying communication protocols, such as line-of-sight (LoS), enables efficient planning with smaller agent groups, further enhancing scalability. These results highlight PRISM’s suitability for complex, dynamic environments, particularly those requiring high coordination, such as search-and-rescue scenarios.

\subsection{Experimental Setup}

For our experiments, we use environments and scenarios from a well-established MAPF benchmark database \cite{mapf-benchmarks}, which provides diverse environments with varied topologies. We conduct two experiments to evaluate PRISM's robustness and scalability, demonstrating its soundness and potential for real-world dynamic, complex settings. 

The robustness experiment tests PRISM's performance under varying network connectivity and compares it to CBS, a centralized baseline. Proximity ranges of `min,' 0.1, 0.15, 0.2, and `full' are used, where each range defines the proximity diameter as a fraction of the environment's longer dimension. For example, a range of 0.1 corresponds to a diameter that is 10\% of this dimension. The `min' range is the smallest viable diameter (4 grid cells) to avoid collisions, while the `full' range ensures a single, fully connected network across the entire system. 

The robustness experiment is conducted in the `random-64-64-20' environment, depicted in Figure \ref{fig:exp_setup}, which contains randomly scattered obstacles occupying 20\% of the space. This setup evaluates PRISM's coordination under general spatial conditions. Performance is assessed across 25 scenarios with randomly sampled motion tasks using three metrics: runtime, success rate, and cost (measured as the sum of costs). PRISM's reported runtime is the total planning time across all agents, capped at 100 minutes. If planning exceeds 100 minutes, it is reported as 100 minutes to reflect efficiency under prolonged computation. Both PRISM and CBS terminate and indicate infeasibility when no solution is found. Since scenarios are derived from valid tasks, success rates only account for cases where runtime exceeded the 100-minute cap, providing insight into PRISM's effectiveness and efficiency under challenging conditions.

The scalability experiment evaluates PRISM's performance as agent and task counts increase, compared to Token Passing with Task Swaps (TPTS) \cite{mlkk-lmapffopadt-17}. This experiment uses varying communication ranges with two protocols: proximity and line-of-sight (LoS), illustrated in Figure \ref{fig:exp_setup}. For LoS, visibility is determined by a straight-line connection between agent positions, with agents within a 4-grid-cell diameter range assumed to be visible, ensuring no immediate collisions. Performance is evaluated over 25 randomly sampled scenarios, using runtime and cost as metrics. A two-minute runtime limit is imposed, with scenarios exceeding this limit considered failures. PRISM's success rate includes only scenarios exceeding the runtime limit since it is complete, while TPTS accounts for scenarios that were ill-formed or exceeded the runtime limit.

These tests are conducted in four distinct environments: `ht\_chantry,' `maze-32-32-2,' `room-64-64-8,' and `warehouse-10-20,' chosen for their unique topological characteristics. The `ht\_chantry' environment, a video game map, includes a diverse topological elements, while the `warehouse-10-20' represents a factory layout. The `maze-32-32-2' and `room-64-64-8' environments feature narrow passages and bottlenecks, which ar traditionally challenging for MAPF problems. This variety enables a thorough evaluation of PRISM's adaptability to different settings and its performance under varying topological constraints. Additionally, plots showing the number of info packets per agent over time are included to illustrate PRISM's ability to minimize communication overhead while maintaining effective coordination. 

In the scalability experiments, tasks are assigned using a heuristic-driven allocator similar to that in TPTS. Tasks are allocated based on Manhattan distance between the agent's position and the task's start. When an agent completes a task, the allocator either assigns a new task or swaps tasks between agents if it improves the overall plan. This dynamic reassignment ensures PRISM agents remain efficient and adaptable throughout the solving process.

\subsection{Token Passing with Task Swaps}

Token Passing with Task Swaps (TPTS) is an online, decentralized, priority-based solver for the MT-MAPF problem. The team holds a token containing the current task allocation plan, paths of prior token-holding agents, and the plan cost. Only the token holder can plan or modify its path, ensuring localized changes. During its turn, the token holder plans its path while avoiding the stored paths of other agents, acting as a priority-based planner. When an agent completes its task, it requests the token, which is passed sequentially in the order requests are made.

The token holder evaluates task swaps using a heuristic based on the distance from its current position to a task's start. If no improvement is possible, it is assigned a new task and plans its path accordingly. The updated path is saved in the token, which is then passed to the next agent. This process continues until all agents have received the token and completed their tasks. For further details, see \cite{mlkk-lmapffopadt-17}.

TPTS relies on two key assumptions: motion tasks are well-formed, and all agents can communicate. A well-formed instance ensures that paths between any start or goal do not intersect other starts or goals. While this assumption is feasible in open environments like warehouses, it is harder to satisfy in constrained spaces with narrow passages, where TPTS cannot resolve deadlocks caused by stationary agents. Additionally, TPTS assumes unrestricted communication among agents, limiting its applicability to real-world scenarios where communication may be constrained.

We compare PRISM with TPTS because TPTS exemplifies an online, decentralized algorithm that avoids deadlocks under specific assumptions. In our experiments, we relax TPTS's well-formed tasks assumption to evaluate its performance in scenarios where task distributions and environmental constraints may cause deadlocks. However, we retain TPTS's unrestricted communication assumption, aligning with its design to focus on trade-offs between communication models and planning approaches. This comparison highlights the contrast between TPTS's simplicity and efficiency, relying on priority-based methods with unrestricted communication, and PRISM's robustness and flexibility, which handle constrained communication and achieve completeness even in the presence of deadlocks.

\begin{table*}[t]
\centering
\caption{Scalability: Communication Range and Maximum Number of Tasks Solved}
{
\resizebox{.8\textwidth}{!}{
\begin{tabular}{c|c|c|c|c|c|}
\cline{2-6}
 & Comm & \# of  & PRISM w/ Proximity & PRISM w/ Line-of-Sight & Token Passing w/ Task Swaps \\
 & Range & Agents & Max \# of Solved Tasks & Max \# of Solved Tasks & Max \# of Solved Tasks \\ \hline
\multicolumn{1}{|c|}{\multirow{3}{*}{ht\_chantry}} & \multirow{3}{*}{0.05} & 10 & \textbf{180} & \textbf{180} & 130 \\
\multicolumn{1}{|c|}{} &  & 20 & \textbf{125} & \textbf{125} & 65 \\
\multicolumn{1}{|c|}{} &  & 30 & 70 & \textbf{90} & 55 \\ \hline
\multicolumn{1}{|c|}{\multirow{3}{*}{maze-32-32-2}} & \multirow{3}{*}{0.15} & 10 & 165 & \textbf{215} & 85 \\
\multicolumn{1}{|c|}{} &  & 15 & 80 & \textbf{120} & 60 \\
\multicolumn{1}{|c|}{} &  & 20 & 50 & \textbf{65} & 45 \\ \hline
\multicolumn{1}{|c|}{\multirow{3}{*}{room-64-64-8}} & \multirow{3}{*}{0.10} & 10 & 265 & \textbf{325} & 265 \\
\multicolumn{1}{|c|}{} &  & 20 & 120 & \textbf{150} & 210 \\
\multicolumn{1}{|c|}{} &  & 30 & \textbf{120} & 115 & 110 \\ \hline
\multicolumn{1}{|c|}{\multirow{3}{*}{warehouse-10-20}} & \multirow{3}{*}{0.05} & 10 & 495 & 495 & \textbf{1000} \\
\multicolumn{1}{|c|}{} &  & 20 & 520 & 575 & \textbf{895 }\\
\multicolumn{1}{|c|}{} &  & 30 & 275 & \textbf{315} & 255 \\ \hline
\end{tabular}
}
}
\label{exp:scale_setup}
\end{table*}

\begin{figure}[t]
    \centering
    \includegraphics[width=.7\columnwidth]{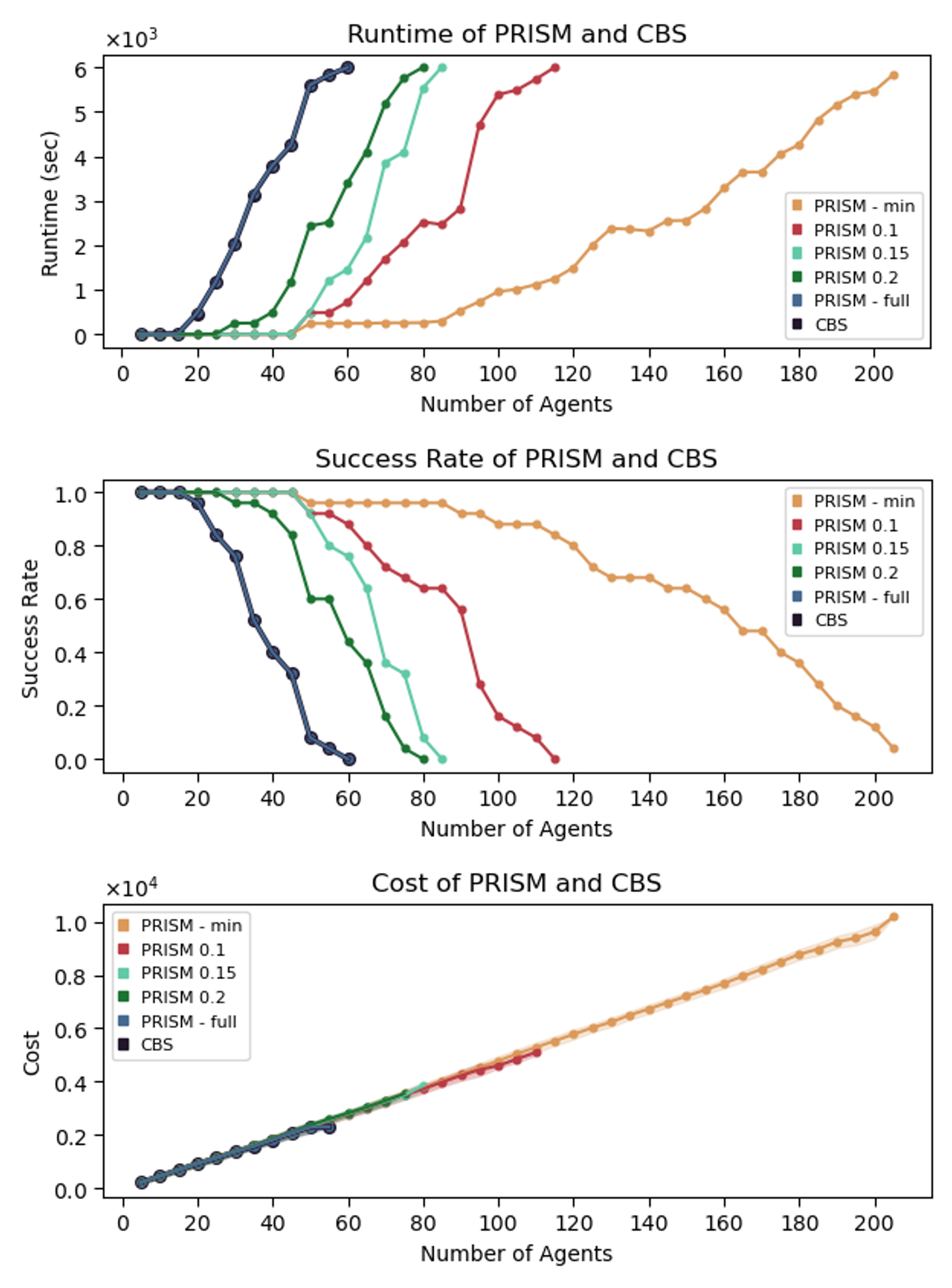}
    \caption{Robustness experiment results for CBS and PRISM with varying proximity ranges for 25 scenarios on the `random-64-64-20' environment. The thicker line shows the results for CBS and PRISM-full to emphasize that fully connected PRISM performs exactly like CBS. The runtimes are capped at 100 minutes.}
    \label{exp:robust}
\end{figure}

\subsection{Robustness Experiment Discussion}

As shown in Figure \ref{exp:robust}, our robustness experiments demonstrate PRISM's strong performance in handling varying levels of network connectivity, consistently outperforming CBS in low-connectivity networks. At `min' connectivity, PRISM solves problems with 3.4 times as many agents as CBS. As connectivity increases, PRISM's performance converges with CBS, particularly at high connectivity levels, where solution costs approach CBS's optimal costs. In low-connectivity settings, smaller agent subsets allow for more frequent but faster replanning, improving scalability compared to CBS. In contrast, high connectivity involves larger agent subsets and less frequent but slower replanning, limiting scalability. 

At `full' connectivity, PRISM achieves performance equivalent to CBS while maintaining completeness with minimal overhead. In lower connectivity networks, solution costs slightly increase but remain close to CBS's optimal costs. Info packets help mitigate degradation by guiding planning in local networks, even when communication is limited. 

The benchmarks tasks are randomly sampled to ensure uniformly distributed optimal path lengths as agent numbers increase, resulting in linear trends in team plan cost. Despite slight degradation in solution quality, PRISM consistently maintains a linear cost trend as the number of agents grows, indicating no sudden declines in performance even with higher communication frequency or larger networks. 

Lastly, while PRISM scales better than CBS, increasing the number of agents still presents challenges. As environments become congested, larger agent groups require more time to resolve. Even in low-connectivity settings, growing agent numbers increase local network sizes, adding computational complexity similar to that faced by centralized solvers.

\begin{table*}[t]
\centering
\caption{Scalability Statistics}
{\smaller
\resizebox{\textwidth}{!}{
\begin{tabular}{cccccccccccc}
 &  &  & \multicolumn{3}{c}{PRISM w/ Proximity} & \multicolumn{3}{c}{PRISM w/ Line-of-Sight} & \multicolumn{3}{c}{Token Passing with Task Swaps} \\
 & $|R|$ & \multicolumn{1}{c|}{$|T|$} & Success & Runtime & \multicolumn{1}{c|}{Cost} & Success & Runtime & \multicolumn{1}{c|}{Cost} & Success & Runtime & \multicolumn{1}{c|}{Cost} \\ \hline
\multicolumn{1}{c|}{\multirow{9}{*}{\centering \arraybackslash \begin{sideways} ht\_chantry \end{sideways}}} & \multirow{3}{*}{10} & \multicolumn{1}{c|}{30} & \textbf{100\%} & \textbf{37.9 $\pm$ 12.9} & \multicolumn{1}{c|}{3764.6 $\pm$ 316.6} & 96\% & 38.0 $\pm$ 12.9 & \multicolumn{1}{c|}{3756.0 $\pm$ 320.2} & \textbf{100\%} & 57.7 $\pm$ 15.7 & \multicolumn{1}{c|}{\textbf{3719.9 $\pm$ 305.3}} \\
\multicolumn{1}{c|}{} &  & \multicolumn{1}{c|}{60} & \textbf{88\%} & \textbf{70.3 $\pm$ 16.3} & \multicolumn{1}{c|}{7539.9 $\pm$ 446.5} & 84\% & 71.8 $\pm$ 16.1 & \multicolumn{1}{c|}{7535.9 $\pm$ 451.6} & 44\% & 83.6 $\pm$ 18.9 & \multicolumn{1}{c|}{\textbf{7335.3 $\pm$ 341.6}} \\
\multicolumn{1}{c|}{} &  & \multicolumn{1}{c|}{90} & \textbf{68\%} & 117.4 $\pm$ 50.2 & \multicolumn{1}{c|}{10766.2 $\pm$ 545.3} & 64\% & 125.8 $\pm$ 58.2 & \multicolumn{1}{c|}{10741.2 $\pm$ 549.7} & 16\% & \textbf{94.0 $\pm$ 9.7} & \multicolumn{1}{c|}{\textbf{10479.5 $\pm$ 170.8}} \\ \cline{2-12} 
\multicolumn{1}{c|}{} & \multirow{3}{*}{20} & \multicolumn{1}{c|}{30} & 76\% & \textbf{37.4 $\pm$ 11.6} & \multicolumn{1}{c|}{3405.9 $\pm$ 299.9} & 76\% & 37.5 $\pm$ 11.5 & \multicolumn{1}{c|}{3406.2 $\pm$ 300.0} & \textbf{88\%} & 69.7 $\pm$ 21.9 & \multicolumn{1}{c|}{\textbf{3373.7 $\pm$ 312.9}} \\
\multicolumn{1}{c|}{} &  & \multicolumn{1}{c|}{45} & \textbf{76\%} & 64.2 $\pm$ 28.8 & \multicolumn{1}{c|}{5436.8 $\pm$ 476.3} & \textbf{76\%} & \textbf{63.4 $\pm$ 26.4} & \multicolumn{1}{c|}{5436.8 $\pm$ 476.3} & 32\% & 99.9 $\pm$ 11.3 & \multicolumn{1}{c|}{\textbf{5146.4 $\pm$ 506.3}} \\
\multicolumn{1}{c|}{} &  & \multicolumn{1}{c|}{60} & \textbf{68\%} & 88.3 $\pm$ 36.3 & \multicolumn{1}{c|}{7144.5 $\pm$ 489.2} & \textbf{68\%} & 89.0 $\pm$ 36.4 & \multicolumn{1}{c|}{7144.8 $\pm$ 489.2} & 12\% & \textbf{84.7 $\pm$ 27.8} & \multicolumn{1}{c|}{\textbf{7069.3 $\pm$ 278.6}} \\ \cline{2-12} 
\multicolumn{1}{c|}{} & \multirow{3}{*}{30} & \multicolumn{1}{c|}{30} & 80\% & 56.8 $\pm$ 51.8 & \multicolumn{1}{c|}{\textbf{2858.2 $\pm$ 288.8}} & 80\% & 55.6 $\pm$ 51.3 & \multicolumn{1}{c|}{2858.3 $\pm$ 288.8} & \textbf{100\%} & \textbf{27.0 $\pm$ 8.8} & \multicolumn{1}{c|}{2950.2 $\pm$ 289.1} \\
\multicolumn{1}{c|}{} &  & \multicolumn{1}{c|}{40} & 64\% & 69.8 $\pm$ 26.5 & \multicolumn{1}{c|}{4336.3 $\pm$ 422.8} & 64\% & \textbf{69.5 $\pm$ 25.3} & \multicolumn{1}{c|}{4336.4 $\pm$ 422.6} & \textbf{76\%} & 81.3 $\pm$ 18.8 & \multicolumn{1}{c|}{\textbf{4317.8 $\pm$ 410.4}} \\
\multicolumn{1}{c|}{} &  & \multicolumn{1}{c|}{50} & \textbf{60\%} & 90.7 $\pm$ 29.7 & \multicolumn{1}{c|}{5688.6 $\pm$ 364.2} & \textbf{60\%} & \textbf{89.7 $\pm$ 28.8} & \multicolumn{1}{c|}{5691.3 $\pm$ 362.1} & 8\% & 91.7 $\pm$ 23.1 & \multicolumn{1}{c|}{\textbf{5072.5 $\pm$ 89.8}} \\ \hline
\multicolumn{1}{c|}{\multirow{9}{*}{\centering \arraybackslash \begin{sideways} maze-32-32-2 \end{sideways}}} & \multirow{3}{*}{10} & \multicolumn{1}{c|}{30} & 92\% & \textbf{24.1 $\pm$ 15.0} & \multicolumn{1}{c|}{2319.8 $\pm$ 187.6} & \textbf{96\%} & 25.8 $\pm$ 36.0 & \multicolumn{1}{c|}{2301.0 $\pm$ 198.2} & 60\% & 12.4 $\pm$ 5.8 & \multicolumn{1}{c|}{\textbf{2290.4 $\pm$ 180.1}} \\
\multicolumn{1}{c|}{} &  & \multicolumn{1}{c|}{50} & 88\% & 40.4 $\pm$ 25.8 & \multicolumn{1}{c|}{3804.1 $\pm$ 248.0} & \textbf{96\%} & 34.1 $\pm$ 23.5 & \multicolumn{1}{c|}{\textbf{3788.8 $\pm$ 233.3}} & 32\% & \textbf{19.7 $\pm$ 6.9} & \multicolumn{1}{c|}{4032.2 $\pm$ 242.5} \\
\multicolumn{1}{c|}{} &  & \multicolumn{1}{c|}{70} & \textbf{88\%} & 70.9 $\pm$ 51.0 & \multicolumn{1}{c|}{\textbf{5071.6 $\pm$ 307.9}} & \textbf{88\%} & 49.3 $\pm$ 30.2 & \multicolumn{1}{c|}{5099.1 $\pm$ 303.8} & 16\% & \textbf{24.6 $\pm$ 11.5} & \multicolumn{1}{c|}{5310.2 $\pm$ 477.4} \\ \cline{2-12} 
\multicolumn{1}{c|}{} & \multirow{3}{*}{15} & \multicolumn{1}{c|}{20} & 88\% & 35.6 $\pm$ 31.6 & \multicolumn{1}{c|}{1339.8 $\pm$ 147.7} & \textbf{96\%} & 27.8 $\pm$ 26.0 & \multicolumn{1}{c|}{\textbf{1330.4 $\pm$ 141.8}} & 76\% & \textbf{8.2 $\pm$ 3.9} & \multicolumn{1}{c|}{1395.1 $\pm$ 127.3} \\
\multicolumn{1}{c|}{} &  & \multicolumn{1}{c|}{30} & 80\% & 44.8 $\pm$ 29.5 & \multicolumn{1}{c|}{2248.4 $\pm$ 228.7} & \textbf{88\%} & 42.4 $\pm$ 43.2 & \multicolumn{1}{c|}{2207.0 $\pm$ 181.7} & 32\% & \textbf{16.7 $\pm$ 7.2} & \multicolumn{1}{c|}{\textbf{2202.9 $\pm$ 252.7}} \\
\multicolumn{1}{c|}{} &  & \multicolumn{1}{c|}{40} & 72\% & 78.8 $\pm$ 61.0 & \multicolumn{1}{c|}{3055.5 $\pm$ 161.4} & \textbf{84\%} & 68.9 $\pm$ 65.6 & \multicolumn{1}{c|}{\textbf{3015.1 $\pm$ 171.6}} & 16\% & \textbf{25.2 $\pm$ 6.0} & \multicolumn{1}{c|}{3134.2 $\pm$ 179.9} \\ \cline{2-12} 
\multicolumn{1}{c|}{} & \multirow{3}{*}{20} & \multicolumn{1}{c|}{20} & 84\% & 47.5 $\pm$ 51.4 & \multicolumn{1}{c|}{1145.5 $\pm$ 117.0} & \textbf{88\%} & 36.4 $\pm$ 48.3 & \multicolumn{1}{c|}{\textbf{1145.4 $\pm$ 110.1}} & 80\% & \textbf{3.2 $\pm$ 0.8} & \multicolumn{1}{c|}{1228.0 $\pm$ 130.7} \\
\multicolumn{1}{c|}{} &  & \multicolumn{1}{c|}{30} & 44\% & 59.4 $\pm$ 29.5 & \multicolumn{1}{c|}{\textbf{2095.3 $\pm$ 224.7}} & \textbf{60\%} & 47.3 $\pm$ 48.3 & \multicolumn{1}{c|}{2114.5 $\pm$ 205.5} & 32\% & \textbf{19.7 $\pm$ 10.3} & \multicolumn{1}{c|}{2179.0 $\pm$ 280.7} \\
\multicolumn{1}{c|}{} &  & \multicolumn{1}{c|}{40} & 12\% & 90.8 $\pm$ 35.0 & \multicolumn{1}{c|}{\textbf{2886.0 $\pm$ 341.3}} & \textbf{44\%} & 95.1 $\pm$ 54.1 & \multicolumn{1}{c|}{2942.3 $\pm$ 204.2} & 12\% & \textbf{26.2 $\pm$ 12.2} & \multicolumn{1}{c|}{2921.0 $\pm$ 369.5} \\ \hline
\multicolumn{1}{c|}{\multirow{9}{*}{\centering \arraybackslash \begin{sideways} room-64-64-8 \end{sideways}}} & \multirow{3}{*}{10} & \multicolumn{1}{c|}{30} & 96\% & 20.9 $\pm$ 15.6 & \multicolumn{1}{c|}{\textbf{2434.4 $\pm$ 180.3}} & \textbf{96\%} & 21.4 $\pm$ 16.4 & \multicolumn{1}{c|}{2435.0 $\pm$ 180.6} & 92\% & \textbf{15.9 $\pm$ 9.6} & \multicolumn{1}{c|}{2438.6 $\pm$ 214.5} \\
\multicolumn{1}{c|}{} &  & \multicolumn{1}{c|}{60} & 96\% & 38.7 $\pm$ 39.3 & \multicolumn{1}{c|}{\textbf{4759.0 $\pm$ 319.8}} & \textbf{96\%} & 34.4 $\pm$ 19.9 & \multicolumn{1}{c|}{4762.7 $\pm$ 322.2} & 64\% & \textbf{24.2 $\pm$ 7.5} & \multicolumn{1}{c|}{4771.8 $\pm$ 319.7} \\
\multicolumn{1}{c|}{} &  & \multicolumn{1}{c|}{90} & 88\% & 69.7 $\pm$ 61.7 & \multicolumn{1}{c|}{\textbf{6941.2 $\pm$ 368.2}} & \textbf{92\%} & 64.1 $\pm$ 50.3 & \multicolumn{1}{c|}{6956.2 $\pm$ 361.8} & 36\% & \textbf{24.8 $\pm$ 5.3} & \multicolumn{1}{c|}{6895.3 $\pm$ 346.8} \\ \cline{2-12} 
\multicolumn{1}{c|}{} & \multirow{3}{*}{20} & \multicolumn{1}{c|}{30} & 96\% & 50.6 $\pm$ 69.5 & \multicolumn{1}{c|}{2234.6 $\pm$ 162.7} & \textbf{96\%} & 26.2 $\pm$ 20.0 & \multicolumn{1}{c|}{2235.6 $\pm$ 162.5} & 84\% & \textbf{21.7 $\pm$ 12.9} & \multicolumn{1}{c|}{\textbf{2193.9 $\pm$ 140.3}} \\
\multicolumn{1}{c|}{} &  & \multicolumn{1}{c|}{45} & \textbf{92\%} & 56.7 $\pm$ 51.8 & \multicolumn{1}{c|}{3520.6 $\pm$ 266.5} & 88\% & 44.9 $\pm$ 31.7 & \multicolumn{1}{c|}{3521.4 $\pm$ 271.4} & 68\% & \textbf{42.0 $\pm$ 25.5} & \multicolumn{1}{c|}{\textbf{3466.9 $\pm$ 219.4}} \\
\multicolumn{1}{c|}{} &  & \multicolumn{1}{c|}{60} & 72\% & 75.5 $\pm$ 40.6 & \multicolumn{1}{c|}{\textbf{4639.9 $\pm$ 253.2}} & \textbf{76\%} & 55.3 $\pm$ 28.7 & \multicolumn{1}{c|}{4644.2 $\pm$ 238.8} & 52\% & \textbf{40.0 $\pm$ 21.2} & \multicolumn{1}{c|}{4645.0 $\pm$ 322.4} \\ \cline{2-12} 
\multicolumn{1}{c|}{} & \multirow{3}{*}{30} & \multicolumn{1}{c|}{30} & 64\% & 58.4 $\pm$ 64.7 & \multicolumn{1}{c|}{1913.1 $\pm$ 157.9} & 72\% & 53.4 $\pm$ 69.6 & \multicolumn{1}{c|}{\textbf{1898.3 $\pm$ 155.1}} & \textbf{92\%} & \textbf{6.4 $\pm$ 1.6} & \multicolumn{1}{c|}{1913.8 $\pm$ 144.1} \\
\multicolumn{1}{c|}{} &  & \multicolumn{1}{c|}{40} & 52\% & 67.4 $\pm$ 57.5 & \multicolumn{1}{c|}{2925.0 $\pm$ 225.0} & 64\% & 56.7 $\pm$ 52.5 & \multicolumn{1}{c|}{2889.8 $\pm$ 219.1} & \textbf{76\%} & \textbf{37.1 $\pm$ 24.1} & \multicolumn{1}{c|}{\textbf{2846.5 $\pm$ 190.6}} \\
\multicolumn{1}{c|}{} &  & \multicolumn{1}{c|}{50} & 32\% & 85.3 $\pm$ 55.0 & \multicolumn{1}{c|}{\textbf{3647.8 $\pm$ 317.3}} & 44\% & 102.8 $\pm$ 68.5 & \multicolumn{1}{c|}{3711.0 $\pm$ 310.9} & \textbf{52\%} & \textbf{45.3 $\pm$ 14.8} & \multicolumn{1}{c|}{3654.5 $\pm$ 239.0} \\ \hline
\multicolumn{1}{c|}{\multirow{9}{*}{\centering \arraybackslash \begin{sideways} warehouse-10-20 \end{sideways}}} & \multirow{3}{*}{10} & \multicolumn{1}{c|}{80} & 92\% & 0.6 $\pm$ 0.3 & \multicolumn{1}{c|}{9548.2 $\pm$ 436.5} & 92\% & \textbf{0.5 $\pm$ 0.3} & \multicolumn{1}{c|}{9548.2 $\pm$ 436.4} & \textbf{100\%} & 8.6 $\pm$ 11.0 & \multicolumn{1}{c|}{\textbf{9504.7 $\pm$ 434.1}} \\
\multicolumn{1}{c|}{} &  & \multicolumn{1}{c|}{160} & 92\% & \textbf{0.9 $\pm$ 0.3} & \multicolumn{1}{c|}{16764.4 $\pm$ 390.4} & 92\% & \textbf{0.9 $\pm$ 0.3} & \multicolumn{1}{c|}{\textbf{16764.7 $\pm$ 390.2}} & \textbf{96\%} & 7.1 $\pm$ 7.4 & \multicolumn{1}{c|}{16810.5 $\pm$ 450.8} \\
\multicolumn{1}{c|}{} &  & \multicolumn{1}{c|}{240} & 92\% & 1.4 $\pm$ 0.6 & \multicolumn{1}{c|}{\textbf{24410.2 $\pm$ 480.7}} & 92\% & \textbf{1.3 $\pm$ 0.3} & \multicolumn{1}{c|}{24411.6 $\pm$ 481.7} & \textbf{96\%} & 11.6 $\pm$ 11.3 & \multicolumn{1}{c|}{24572.7 $\pm$ 490.5} \\ \cline{2-12} 
\multicolumn{1}{c|}{} & \multirow{3}{*}{20} & \multicolumn{1}{c|}{80} & 92\% & \textbf{0.9 $\pm$ 0.3} & \multicolumn{1}{c|}{9176.7 $\pm$ 373.5} & \textbf{96\%} & \textbf{0.9 $\pm$ 0.3} & \multicolumn{1}{c|}{\textbf{9169.8 $\pm$ 367.0}} & 80\% & 16.5 $\pm$ 24.0 & \multicolumn{1}{c|}{9182.8 $\pm$ 452.5} \\
\multicolumn{1}{c|}{} &  & \multicolumn{1}{c|}{160} & 92\% & 1.6 $\pm$ 0.5 & \multicolumn{1}{c|}{16568.5 $\pm$ 403.6} & \textbf{96\%} & \textbf{1.5 $\pm$ 0.4} & \multicolumn{1}{c|}{\textbf{16559.7 $\pm$ 396.1}} & 72\% & 17.6 $\pm$ 14.4 & \multicolumn{1}{c|}{16578.7 $\pm$ 421.1} \\
\multicolumn{1}{c|}{} &  & \multicolumn{1}{c|}{240} & 80\% & 2.2 $\pm$ 0.5 & \multicolumn{1}{c|}{24176.3 $\pm$ 485.1} & \textbf{84\%} & \textbf{2.1 $\pm$ 0.5} & \multicolumn{1}{c|}{\textbf{24174.9 $\pm$ 477.2}} & 60\% & 28.2 $\pm$ 14.7 & \multicolumn{1}{c|}{24511.3 $\pm$ 490.3} \\ \cline{2-12} 
\multicolumn{1}{c|}{} & \multirow{3}{*}{30} & \multicolumn{1}{c|}{60} & \textbf{88\%} & \textbf{1.3 $\pm$ 0.5} & \multicolumn{1}{c|}{7061.4 $\pm$ 356.8} & \textbf{88\%} & 1.4 $\pm$ 0.6 & \multicolumn{1}{c|}{7062.8 $\pm$ 356.9} & 72\% & 22.3 $\pm$ 21.7 & \multicolumn{1}{c|}{\textbf{6936.4 $\pm$ 303.6}} \\
\multicolumn{1}{c|}{} &  & \multicolumn{1}{c|}{120} & \textbf{84\%} & 2.0 $\pm$ 0.5 & \multicolumn{1}{c|}{\textbf{12899.9 $\pm$ 397.2}} & \textbf{84\%} & \textbf{1.9 $\pm$ 0.4} & \multicolumn{1}{c|}{12900.7 $\pm$ 391.4} & 32\% & 45.5 $\pm$ 39.1 & \multicolumn{1}{c|}{12924.0 $\pm$ 374.8} \\
\multicolumn{1}{c|}{} &  & \multicolumn{1}{c|}{180} & 72\% & 8.9 $\pm$ 26.0 & \multicolumn{1}{c|}{18881.6 $\pm$ 457.7} & \textbf{76\%} & \textbf{2.6 $\pm$ 0.4} & \multicolumn{1}{c|}{18917.8 $\pm$ 470.8} & 24\% & 27.3 $\pm$ 15.9 & \multicolumn{1}{c|}{\textbf{18821.3 $\pm$ 313.6}} \\ \hline
\end{tabular}
}
}
\label{exp:scale}
\end{table*}

\subsection{Scalability Experiment Discussion}

\subsubsection{General Performance} 

The scalability experiment results, summarized in Tables \ref{exp:scale_setup} and \ref{exp:scale}, compare PRISM's proximity (PRISM Prox) and line-of-sight (PRISM LoS) protocols against TPTS across various environments. Table \ref{exp:scale_setup} outlines the communication ranges for each PRISM protocol and the maximum number of tasks (Max $|T|$) solved for different team sizes ($|R|$). These ranges were selected to highlight performance differences between the protocols, rather than to optimize PRISM's performance. Table \ref{exp:scale} compares PRISM and TPTS on success rate (SR), average runtime and cost (with standard deviations) for varying team and task sizes. The best results for each configuration are bolded. However, it is important to note that as success rates drop for larger, more challenging scenarios, the average runtime and cost values may skew toward easier scenarios. 

PRISM effectively scales with increasing tasks, despite longer planning times caused by additional iterations. Its sequential task allocation ensures agents focus on one task at a time, maintaining efficiency even as task specifications evolve. 

The scalability experiments also demonstrate PRISM's adaptability to different topologies. In open environments, like `ht\_chantry' and `warehouse-10-20,' performance variability increases with more agents, while in constrained environments like `room-64-64-8' and `maze-32-32-2,' PRISM excels by handling narrow passages and bottlenecks effectively. By limiting communication, PRISM focuses on smaller subsets of agents, reducing conflicts and ensuring manageable subproblem sizes during planning iterations. 

In these constrained settings, runtime grows linearly with the number of agents, showcasing PRISM's scalability. This linear growth, rather than exponential escalation, highlights PRISM's ability to efficiently handle increasing complexity by leveraging smaller subproblems facilitated by communication constraints. Overall, PRISM demonstrates robust and scalable performance across diverse environmental conditions. 

\subsubsection{Communication Protocol Comparison}

PRISM demonstrates strong performance under constrained communication conditions, particularly with the line-of-sight (LoS) protocol. LoS forms smaller, more manageable agent groups during planning, enabling PRISM to efficiently handle more tasks and agents in narrow passage environments, where fewer agents require simultaneous coordination. As a result, PRISM with LoS achieves better scalability in these settings compared to proximity-based communication. 

In open spaces, the performance of both protocols is similar in terms of success rates and runtime. However, in narrow environments, PRISM with LoS shows a slight advantage, benefiting from smaller agent groupings. Despite reduced network connectivity under LoS, no significant differences in solution quality are observed between the two protocols. This consistency highlights PRISM's robustness and adaptability, making it well-suited for scenarios with limited direct communication. 

\begin{figure*}[t!]
    \centering
    \begin{subfigure}[t]{\textwidth}
        \centering
        \includegraphics[width=\textwidth]{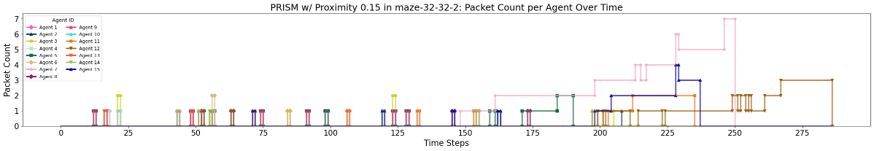}
        \caption{}
        \label{fig:maze_packets-prox}
    \end{subfigure}
    ~ 
    \begin{subfigure}[t]{\textwidth}
        \centering
        \includegraphics[width=\textwidth]{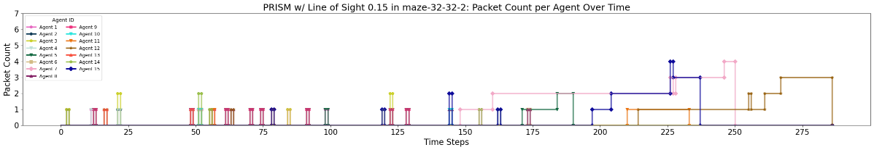}
        \caption{}
        \label{fig:room_packet-los}
    \end{subfigure}
    ~ 
    \begin{subfigure}[t]{\textwidth}
        \centering
        \includegraphics[width=\textwidth]{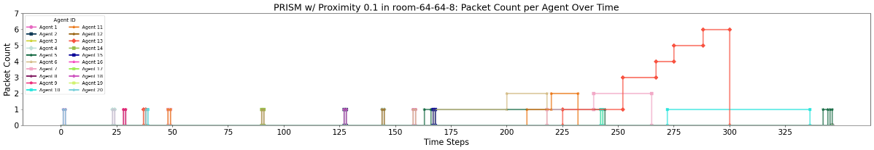}
        \caption{}
        \label{fig:room_packet-prox}
    \end{subfigure}
    ~ 
    \begin{subfigure}[t]{\textwidth}
        \centering
        \includegraphics[width=\textwidth]{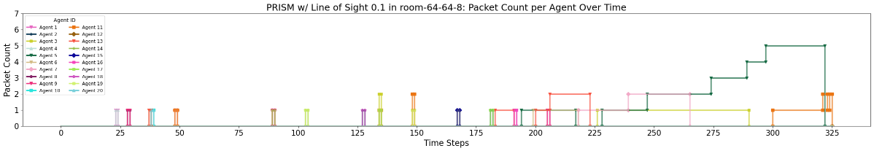}
        \caption{}
        \label{fig:room_packet-los}
    \end{subfigure}
    \caption{This figure depicts the packet count per agent over time, with each agent represented by a unique color. Markers are included to highlight fluctuations in packet counts for improved visibility. Results are shown for the `maze-32-32' map with 15 agents solving 30 tasks using (a) a 0.15 proximity range and (b) a 0.15 line-of-sight (LoS) range. Similarly, the `room-64-64-8' map is evaluated with 20 agents completing 60 tasks, using (c) a 0.10 proximity range and (d) a 0.10 LoS range.}
    \label{exp:packets}
\end{figure*}

\subsubsection{PRISM vs. TPTS}

When comparing PRISM to TPTS, PRISM exhibits more graceful performance degradation as the number of tasks and agents increases. TPTS relies on well-formed problems, and when this condition is unmet, it cannot terminate on its own and requires external intervention. In contrast, PRISM handles all problem types effectively, ensuring reliable performance across diverse scenarios. However, when TPTS can solve a problem, it is typically faster due to its low-coordination, priority-based planning scheme.

TPTS performs well in environments like `ht\_chantry' and `warehouse-10-20,' where tasks are more likely to be well-formed. In such settings, its scalability and faster planning times outperform PRISM. Conversely, PRISM excels in narrow passage environments that require higher coordination, consistently providing solutions even in challenging scenarios. For example, in a maze environment, PRISM scales to 2.5 times as many tasks as TPTS for a 10-agent team, making it ideal for complex scenarios where robust coordination is essential.

Solution quality also varies between the two approaches. In open environments, TPTS often produces lower-cost solutions by leveraging global token-based information. In contrast, PRISM, which relies on local information, yields slightly higher costs in these settings. However, in narrow environments, PRISM outperforms TPTS by achieving lower solution costs due to its superior coordination, while TPTS's reduced coordination often leads to inefficient routes or unnecessary delays.

PRISM’s performance is influenced by network connectivity. By limiting path resolution to nearby agents within communication range, PRISM supports rapid replanning and maintains efficiency. Its integration of constraint-based search and info packets ensures high-quality paths in a decentralized framework. While performance bottlenecks still occur as the number of agents grows, these are less pronounced compared to centralized methods. In open environments, however, simpler decentralized approaches like TPTS may offer better scalability and faster runtimes, indicating that PRISM is best suited for scenarios requiring high coordination, whereas TPTS is better for less complex settings.

For applications such as search-and-rescue in cluttered or constrained environments, PRISM’s ability to handle high levels of coordination makes it the preferred choice. For traditional warehouse tasks like pickup-and-delivery, where tasks are well-formed and space is ample, TPTS is likely to outperform PRISM in both scalability and efficiency.

\subsubsection{Info Packet Counts}
In addition to runtime and cost results, we present examples of how frequently agents utilize info packets during planning, shown in Figures \ref{exp:packets}. These plots illustrate the packet count per agent over time for PRISM using line-of-sight and proximity communication protocols. The data shows that PRISM relies minimally on info packets, with most agents retaining only one or two packets briefly. Towards the end of planning, packet counts increase as agents complete their tasks and transition to a resting state. This leads to the accumulation of infinite info packets, which persist until agents finish their tasks. Once all tasks are completed, this accumulation is cleared, as reflected in the sharp drop to zero in the plots.

In the `maze-32-32-2' environment, agents rely more heavily on info packets than in `room-64-64-8,' reflecting the maze's higher coordination demands. The maze’s narrow passages require frequent use of info packets for local planning due to local networks that change more frequently. 

Between the two communication protocols, line-of-sight shows more frequent info packet use early in planning but fewer infinite flush time packets toward the end. This is due to the limited communication range of line-of-sight, which causes more frequent network changes. However, as tasks are completed, the constrained range prevents significant packet accumulation. Overall, PRISM effectively leverages info packets to enhance planning without agents accumulating or retaining them for long durations.

\begin{figure*}[t!]
    \centering
    \begin{subfigure}[t]{0.3\textwidth}
        \centering
        \includegraphics[width=\textwidth]{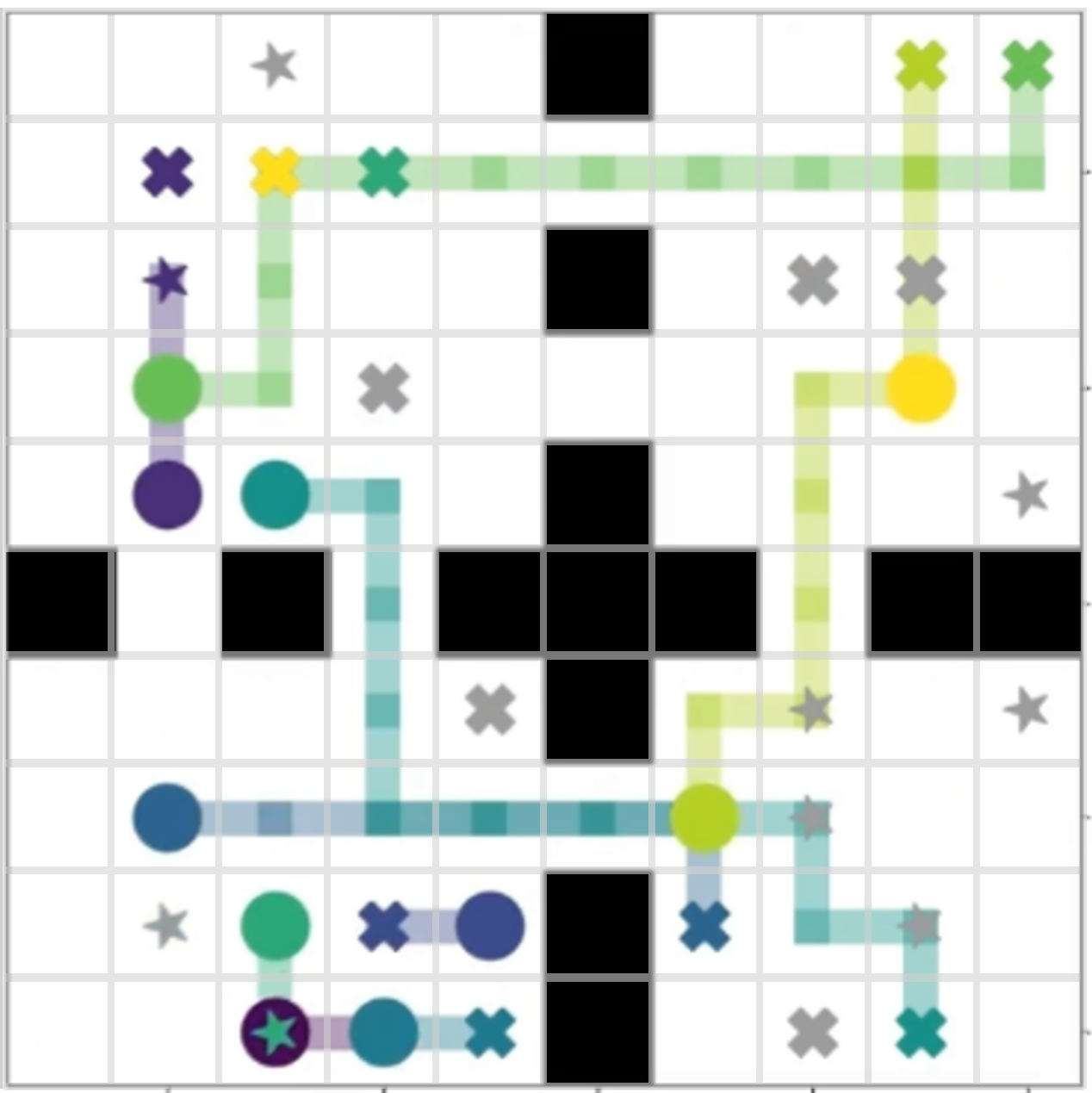}
        \caption{Simulation}
        \label{fig:simulation-exp}
    \end{subfigure}%
    ~ 
    \begin{subfigure}[t]{0.3\textwidth}
        \centering
        \includegraphics[width=\textwidth]{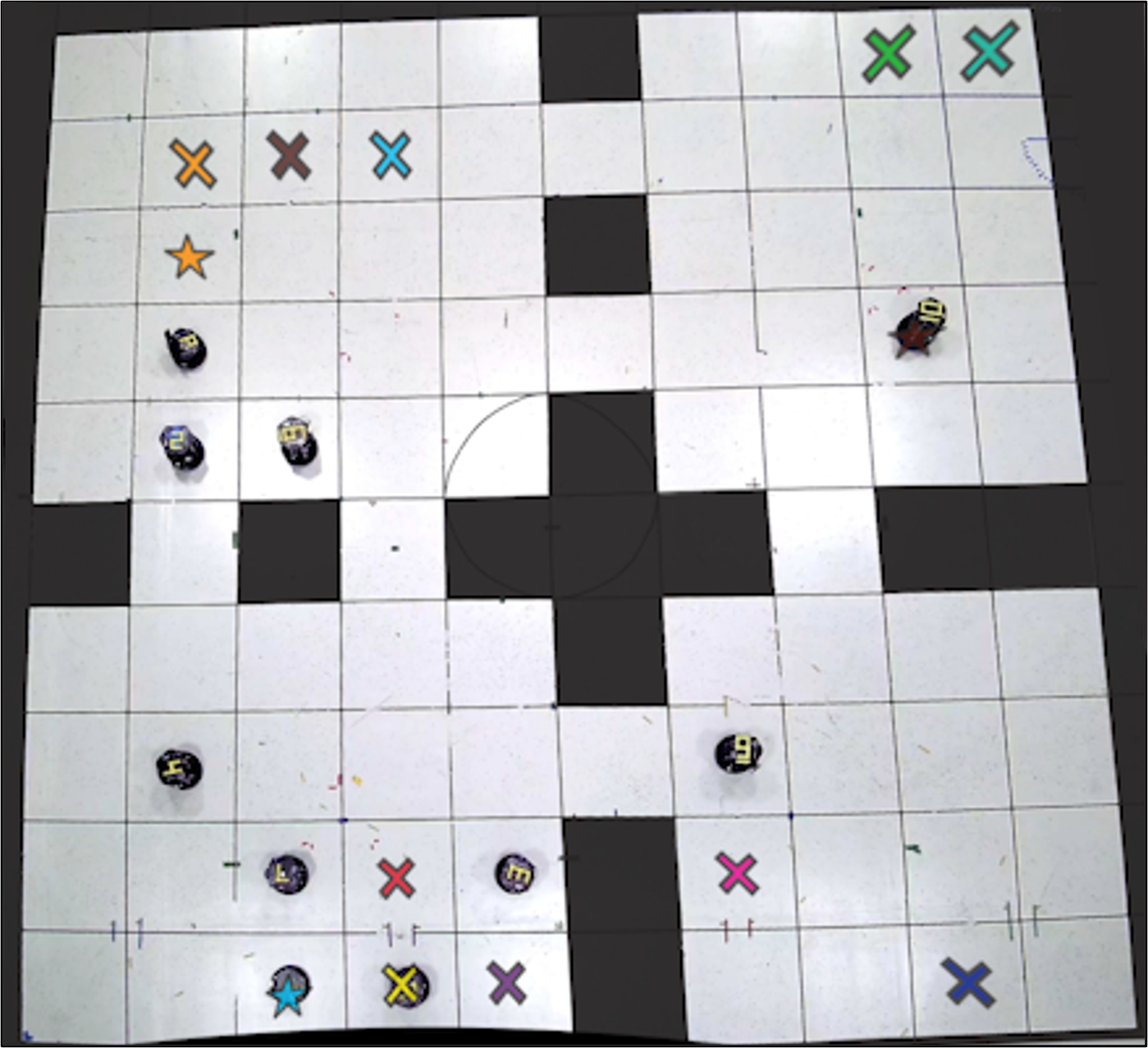}
        \caption{Physical}
        \label{fig:physical-exp}
    \end{subfigure}
    \caption{(a) Simulation of 10 agents (circles) in a small room environment following their planned paths to task start locations (stars) and then to task goal locations (X's). (b) A bird's-eye view from the physical experiment using TurtleBot3 robots, with task start (stars) and goal locations (X's) overlaid on the image for direct comparison..}
    \label{exp:phys_exp}
\end{figure*}

\subsection{Physical System Validation}

To validate PRISM in a real-world setting, we conducted physical experiments in an environment designed to mimic interconnected rooms with narrow bottlenecks, creating challenging navigation scenarios. A shot of our physical experiment and its simulation are shown in Figure \ref{exp:phys_exp}. A team of 10 TurtleBot2 robots was deployed to collectively complete a set of 20 predefined tasks. The environment was structured to test PRISM's ability to manage congestion and deadlocks effectively, particularly in the bottleneck regions where agent coordination is critical. 

The experiments were implemented using the Robot Operating System (ROS), with each TurtleBot running on a Raspberry Pi for onboard computation. A motion capture system was used to provide localization data, enabling accurate real-time feedback for robot control. For navigation, we employed a point stabilization controller to guide the robots efficiently between task locations. These experiments showcased PRISM's ability to coordinate multiple agents in a constrained physical environment, reinforcing its practicality and robustness under real-world conditions.

\section{Conclusion} \label{Conclusion}

In this work, we presented PRISM, a decentralized multi-agent pathfinding framework that combines constraint-based search with selective communication via info packets. By restricting path resolution to interactions with relevant nearby agents, determined by their communication range, PRISM enables efficient and rapid replanning. This approach allows PRISM to generate high-quality paths even in dynamic and complex environments. Notably, we proved that PRISM is both complete and guarantees deadlock avoidance, ensuring reliable operation across a variety of scenarios.

PRISM also demonstrates significant strengths in environments requiring high levels of coordination, such as those characterized by narrow passages or dense task allocations, consistently delivering robust solutions despite increased complexity. In contrast, simpler decentralized methods like TPTS, with lower computational overhead, excel in spacious environments with fewer coordination demands, offering better scalability and faster runtimes. This contrast highlights PRISM's adaptability and reliability in challenging scenarios while acknowledging the advantages of TPTS in less demanding settings. Although PRISM experiences performance bottlenecks as the number of agents increases, these are far less pronounced compared to centralized solvers, emphasizing its ability to scale effectively. 

PRISM’s flexibility as a framework makes it well-suited for diverse applications. For instance, while this work assumes agents rest at their goal positions upon task completion to avoid resource conflicts, PRISM can be adapted to different domains by incorporating predefined parking positions or dynamically sampled resting locations. For example, in warehouse environments, designated parking areas may be required, whereas search-and-rescue scenarios might benefit from dynamic repositioning to address task conflicts. This adaptability highlights PRISM’s potential to tackle a broad spectrum of multi-agent planning challenges across various domains.

Ultimately, PRISM represents a robust framework for adapting constraint-based search algorithms to decentralized systems. Its scalability, adaptability, and reliability make it a strong foundation for solving complex multi-agent planning problems in both structured and dynamic environments.

\vskip 0.2in
\bibliography{references}

\begin{thebibliography}{}

\bibitem[\protect\BCAY{Asama, Ozaki, Itakura, Matsumoto, Ishida,\ \BBA\ Endo}{Asama et~al.}{1991}]{asama1991collision}
Asama, H., Ozaki, K., Itakura, H., Matsumoto, A., Ishida, Y., \BBA\ Endo, I. \BBOP1991\BBCP.
\newblock \BBOQ Collision avoidance among multiple mobile robots based on rules and communication.\BBCQ\
\newblock In {\Bem IROS}, \lowercase{\BVOL}~91, \BPGS\ 1215--1220.

\bibitem[\protect\BCAY{Balch\ \BBA\ Arkin}{Balch\ \BBA\ Arkin}{1998}]{ba-bbfcfmt-98}
Balch, T.\BBACOMMA\  \BBA\ Arkin, R.~C. \BBOP1998\BBCP.
\newblock \BBOQ Behavior-based formation control for multirobot teams\BBCQ\
\newblock {\Bem IEEE transactions on robotics and automation}, {\Bem 14\/}(6), 926--939.

\bibitem[\protect\BCAY{Boyarski, Felner, Stern, Sharon, Shimony, Bezalel,\ \BBA\ Tolpin}{Boyarski et~al.}{2015}]{bfsssbt-icbsfomapf-15}
Boyarski, E., Felner, A., Stern, R., Sharon, G., Shimony, E., Bezalel, O., \BBA\ Tolpin, D. \BBOP2015\BBCP.
\newblock \BBOQ Improved conflict-based search for optimal multi-agent path finding\BBCQ\
\newblock In {\Bem 24th International Joint Conference on Artificial Intelligence, IJCAI 2015}.

\bibitem[\protect\BCAY{Brown, Peltzer, Sehr, Schwager,\ \BBA\ Kochenderfer}{Brown et~al.}{2020}]{bpssk-ostapffmarap-20}
Brown, K., Peltzer, O., Sehr, M.~A., Schwager, M., \BBA\ Kochenderfer, M.~J. \BBOP2020\BBCP.
\newblock \BBOQ Optimal sequential task assignment and path finding for multi-agent robotic assembly planning\BBCQ\
\newblock In {\Bem 2020 IEEE International Conference on Robotics and Automation (ICRA)}, \BPGS\ 441--447. IEEE.

\bibitem[\protect\BCAY{{\v{C}}{\'a}p, Nov{\'a}k, Kleiner,\ \BBA\ Seleck{\`y}}{{\v{C}}{\'a}p et~al.}{2015}]{vcap2015prioritized}
{\v{C}}{\'a}p, M., Nov{\'a}k, P., Kleiner, A., \BBA\ Seleck{\`y}, M. \BBOP2015\BBCP.
\newblock \BBOQ Prioritized planning algorithms for trajectory coordination of multiple mobile robots\BBCQ\
\newblock {\Bem IEEE transactions on automation science and engineering}, {\Bem 12\/}(3), 835--849.

\bibitem[\protect\BCAY{Chan, Stern, Felner,\ \BBA\ Koenig}{Chan et~al.}{2023}]{chan2023greedy}
Chan, S.-H., Stern, R., Felner, A., \BBA\ Koenig, S. \BBOP2023\BBCP.
\newblock \BBOQ Greedy priority-based search for suboptimal multi-agent path finding\BBCQ\
\newblock In {\Bem Proceedings of the International Symposium on Combinatorial Search}, \lowercase{\BVOL}~16, \BPGS\ 11--19.

\bibitem[\protect\BCAY{Choi, Brunet,\ \BBA\ How}{Choi et~al.}{2009}]{choi2009consensus}
Choi, H.-L., Brunet, L., \BBA\ How, J.~P. \BBOP2009\BBCP.
\newblock \BBOQ Consensus-based decentralized auctions for robust task allocation\BBCQ\
\newblock {\Bem IEEE transactions on robotics}, {\Bem 25\/}(4), 912--926.

\bibitem[\protect\BCAY{Cohen, Uras, Kumar,\ \BBA\ Koenig}{Cohen et~al.}{2019}]{cukk-oabsmamp-19}
Cohen, L., Uras, T., Kumar, T., \BBA\ Koenig, S. \BBOP2019\BBCP.
\newblock \BBOQ Optimal and bounded-suboptimal multi-agent motion planning\BBCQ\
\newblock In {\Bem Proceedings of the International Symposium on Combinatorial Search}, \lowercase{\BVOL}~10, \BPGS\ 44--51.

\bibitem[\protect\BCAY{Desaraju\ \BBA\ How}{Desaraju\ \BBA\ How}{2011}]{desaraju2011decentralized}
Desaraju, V.~R.\BBACOMMA\  \BBA\ How, J.~P. \BBOP2011\BBCP.
\newblock \BBOQ Decentralized path planning for multi-agent teams in complex environments using rapidly-exploring random trees\BBCQ\
\newblock In {\Bem 2011 IEEE International Conference on Robotics and Automation}, \BPGS\ 4956--4961. IEEE.

\bibitem[\protect\BCAY{Fox, Burgard, Kruppa,\ \BBA\ Thrun}{Fox et~al.}{2000}]{fbkt-apptcmrl-00}
Fox, D., Burgard, W., Kruppa, H., \BBA\ Thrun, S. \BBOP2000\BBCP.
\newblock \BBOQ A probabilistic approach to collaborative multi-robot localization\BBCQ\
\newblock {\Bem Autonomous robots}, {\Bem 8}, 325--344.

\bibitem[\protect\BCAY{Gui, Yu, Deng, Zhu,\ \BBA\ Yao}{Gui et~al.}{2023}]{gui2023decentralized}
Gui, J., Yu, T., Deng, B., Zhu, X., \BBA\ Yao, W. \BBOP2023\BBCP.
\newblock \BBOQ Decentralized multi-uav cooperative exploration using dynamic centroid-based area partition\BBCQ\
\newblock {\Bem Drones}, {\Bem 7\/}(6), 337.

\bibitem[\protect\BCAY{Halperin, Latombe,\ \BBA\ Wilson}{Halperin et~al.}{1998}]{hlw-agffaptmsa-98}
Halperin, D., Latombe, J.-C., \BBA\ Wilson, R.~H. \BBOP1998\BBCP.
\newblock \BBOQ A general framework for assembly planning: The motion space approach\BBCQ\
\newblock In {\Bem Proceedings of the fourteenth annual symposium on Computational geometry}, \BPGS\ 9--18.

\bibitem[\protect\BCAY{Ho, Geraldes, Gon{\c{c}}alves, Rigault, Sportich, Kubo, Cavazza,\ \BBA\ Prendinger}{Ho et~al.}{2020}]{ho2020decentralized}
Ho, F., Geraldes, R., Gon{\c{c}}alves, A., Rigault, B., Sportich, B., Kubo, D., Cavazza, M., \BBA\ Prendinger, H. \BBOP2020\BBCP.
\newblock \BBOQ Decentralized multi-agent path finding for uav traffic management\BBCQ\
\newblock {\Bem IEEE Transactions on Intelligent Transportation Systems}, {\Bem 23\/}(2), 997--1008.

\bibitem[\protect\BCAY{Hwang, Kim,\ \BBA\ Tomlin}{Hwang et~al.}{2007}]{hwang2007protocol}
Hwang, I., Kim, J., \BBA\ Tomlin, C. \BBOP2007\BBCP.
\newblock \BBOQ Protocol-based conflict resolution for air traffic control\BBCQ\
\newblock {\Bem Air Traffic Control Quarterly}, {\Bem 15\/}(1), 1--34.

\bibitem[\protect\BCAY{Izadi, Gordon,\ \BBA\ Zhang}{Izadi et~al.}{2011}]{izadi2011rule}
Izadi, H.~A., Gordon, B.~W., \BBA\ Zhang, Y. \BBOP2011\BBCP.
\newblock \BBOQ Rule-based cooperative collision avoidance using decentralized model predictive control\BBCQ\
\newblock In {\Bem Infotech@ Aerospace 2011}, \BPG\ 1610.

\bibitem[\protect\BCAY{Kloder\ \BBA\ Hutchinson}{Kloder\ \BBA\ Hutchinson}{2006}]{kh-ppfpimf-06}
Kloder, S.\BBACOMMA\  \BBA\ Hutchinson, S. \BBOP2006\BBCP.
\newblock \BBOQ Path planning for permutation-invariant multirobot formations\BBCQ\
\newblock {\Bem IEEE Transactions on Robotics}, {\Bem 22\/}(4), 650--665.

\bibitem[\protect\BCAY{Lee, Motes, Morales,\ \BBA\ Amato}{Lee et~al.}{2021}]{hccbs}
Lee, H., Motes, J., Morales, M., \BBA\ Amato, N.~M. \BBOP2021\BBCP.
\newblock \BBOQ Parallel hierarchical composition conflict-based search for optimal multi-agent pathfinding\BBCQ\
\newblock {\Bem IEEE Robotics and Automation Letters}, {\Bem 6\/}(4), 7001--7008.

\bibitem[\protect\BCAY{Li, Gange, Harabor, Stuckey, Ma,\ \BBA\ Koenig}{Li et~al.}{2020}]{lghsmk-ntfpsbimapf-20}
Li, J., Gange, G., Harabor, D., Stuckey, P.~J., Ma, H., \BBA\ Koenig, S. \BBOP2020\BBCP.
\newblock \BBOQ New techniques for pairwise symmetry breaking in multi-agent path finding\BBCQ\
\newblock In {\Bem Proceedings of the International Conference on Automated Planning and Scheduling}, \lowercase{\BVOL}~30, \BPGS\ 193--201.

\bibitem[\protect\BCAY{Li, Harabor, Stuckey, Felner, Ma,\ \BBA\ Koenig}{Li et~al.}{2019a}]{lhsfmk-dsfmapfwcbs-19}
Li, J., Harabor, D., Stuckey, P.~J., Felner, A., Ma, H., \BBA\ Koenig, S. \BBOP2019a\BBCP.
\newblock \BBOQ Disjoint splitting for multi-agent path finding with conflict-based search\BBCQ\
\newblock In {\Bem Proceedings of the international conference on automated planning and scheduling}, \lowercase{\BVOL}~29, \BPGS\ 279--283.

\bibitem[\protect\BCAY{Li, Harabor, Stuckey, Ma,\ \BBA\ Koenig}{Li et~al.}{2019b}]{lhsmk-sbcfgbmapf-19}
Li, J., Harabor, D., Stuckey, P.~J., Ma, H., \BBA\ Koenig, S. \BBOP2019b\BBCP.
\newblock \BBOQ Symmetry-breaking constraints for grid-based multi-agent path finding\BBCQ\
\newblock In {\Bem Proceedings of the AAAI conference on artificial intelligence}, \lowercase{\BVOL}~33, \BPGS\ 6087--6095.

\bibitem[\protect\BCAY{Li, Sun, Ma, Felner, Kumar,\ \BBA\ Koenig}{Li et~al.}{2020}]{lsmfkk-maifice-20}
Li, J., Sun, K., Ma, H., Felner, A., Kumar, T., \BBA\ Koenig, S. \BBOP2020\BBCP.
\newblock \BBOQ Moving agents in formation in congested environments\BBCQ\
\newblock In {\Bem Proceedings of the International Symposium on Combinatorial Search}, \lowercase{\BVOL}~11, \BPGS\ 131--132.

\bibitem[\protect\BCAY{Liu, Ma, Li,\ \BBA\ Koenig}{Liu et~al.}{2019}]{lmlk-tappfmapd-19}
Liu, M., Ma, H., Li, J., \BBA\ Koenig, S. \BBOP2019\BBCP.
\newblock \BBOQ Task and path planning for multi-agent pickup and delivery\BBCQ\
\newblock In {\Bem Proceedings of the International Joint Conference on Autonomous Agents and Multiagent Systems (AAMAS)}.

\bibitem[\protect\BCAY{Liu, Wen, Cui, Yang, Cao,\ \BBA\ Liu}{Liu et~al.}{2021}]{lwcycl-mfifadhlatmamt-21}
Liu, S., Wen, L., Cui, J., Yang, X., Cao, J., \BBA\ Liu, Y. \BBOP2021\BBCP.
\newblock \BBOQ Moving forward in formation: A decentralized hierarchical learning approach to multi-agent moving together\BBCQ\
\newblock In {\Bem 2021 IEEE/RSJ International Conference on Intelligent Robots and Systems (IROS)}, \BPGS\ 4777--4784. IEEE.

\bibitem[\protect\BCAY{Ma, Harabor, Stuckey, Li,\ \BBA\ Koenig}{Ma et~al.}{2019}]{pbs}
Ma, H., Harabor, D., Stuckey, P.~J., Li, J., \BBA\ Koenig, S. \BBOP2019\BBCP.
\newblock \BBOQ Searching with consistent prioritization for multi-agent path finding\BBCQ\
\newblock In {\Bem Proceedings of the AAAI conference on artificial intelligence}, \lowercase{\BVOL}~33, \BPGS\ 7643--7650.

\bibitem[\protect\BCAY{Ma, Li, Kumar,\ \BBA\ Koenig}{Ma et~al.}{2017}]{mlkk-lmapffopadt-17}
Ma, H., Li, J., Kumar, T.~S., \BBA\ Koenig, S. \BBOP2017\BBCP.
\newblock \BBOQ Lifelong multi-agent path finding for online pickup and delivery tasks\BBCQ\
\newblock In {\Bem Proceedings of the 16th Conference on Autonomous Agents and MultiAgent Systems}, \BPGS\ 837--845.

\bibitem[\protect\BCAY{Ma, Tovey, Sharon, Kumar,\ \BBA\ Koenig}{Ma et~al.}{2016}]{ma2016multi}
Ma, H., Tovey, C., Sharon, G., Kumar, T., \BBA\ Koenig, S. \BBOP2016\BBCP.
\newblock \BBOQ Multi-agent path finding with payload transfers and the package-exchange robot-routing problem\BBCQ\
\newblock In {\Bem Proceedings of the AAAI Conference on Artificial Intelligence}, \lowercase{\BVOL}~30.

\bibitem[\protect\BCAY{Masehian\ \BBA\ Nejad}{Masehian\ \BBA\ Nejad}{2010}]{masehian2010hierarchical}
Masehian, E.\BBACOMMA\  \BBA\ Nejad, A.~H. \BBOP2010\BBCP.
\newblock \BBOQ A hierarchical decoupled approach for multi robot motion planning on trees\BBCQ\
\newblock In {\Bem 2010 IEEE International Conference on Robotics and Automation}, \BPGS\ 3604--3609. IEEE.

\bibitem[\protect\BCAY{Matoui, Boussaid, Metoui, Frej,\ \BBA\ Abdelkrim}{Matoui et~al.}{2017}]{matoui2017path}
Matoui, F., Boussaid, B., Metoui, B., Frej, G., \BBA\ Abdelkrim, M.~N. \BBOP2017\BBCP.
\newblock \BBOQ Path planning of a group of robots with potential field approach: decentralized architecture\BBCQ\
\newblock {\Bem IFAC-PapersOnLine}, {\Bem 50\/}(1), 11473--11478.

\bibitem[\protect\BCAY{Mikkelsen\ \BBA\ Fumagalli}{Mikkelsen\ \BBA\ Fumagalli}{2023}]{mikkelsen2023distributed}
Mikkelsen, J.~H.\BBACOMMA\  \BBA\ Fumagalli, M. \BBOP2023\BBCP.
\newblock \BBOQ Distributed planning for rigid robot formations using consensus on the transformation of a base configuration\BBCQ\
\newblock In {\Bem 2023 21st International Conference on Advanced Robotics (ICAR)}, \BPGS\ 627--632. IEEE.

\bibitem[\protect\BCAY{Nnaji}{Nnaji}{1993}]{nb-toaraap-93}
Nnaji, B.~O. \BBOP1993\BBCP.
\newblock {\Bem Theory of automatic robot assembly and programming}.
\newblock Springer Science \& Business Media.

\bibitem[\protect\BCAY{Pianpak, Son, Toups~Dugas,\ \BBA\ Yeoh}{Pianpak et~al.}{2019}]{pianpak2019distributed}
Pianpak, P., Son, T.~C., Toups~Dugas, P.~O., \BBA\ Yeoh, W. \BBOP2019\BBCP.
\newblock \BBOQ A distributed solver for multi-agent path finding problems\BBCQ\
\newblock In {\Bem Proceedings of the First International Conference on Distributed Artificial Intelligence}, \BPGS\ 1--7.

\bibitem[\protect\BCAY{Pradhan, Roy,\ \BBA\ Hui}{Pradhan et~al.}{2018}]{pradhan2018motion}
Pradhan, B., Roy, D.~S., \BBA\ Hui, N.~B. \BBOP2018\BBCP.
\newblock \BBOQ Motion planning and coordination of multi-agent systems\BBCQ\
\newblock {\Bem International Journal of Computational Vision and Robotics}, {\Bem 8\/}(5), 492--508.

\bibitem[\protect\BCAY{Purwin, D’Andrea,\ \BBA\ Lee}{Purwin et~al.}{2008}]{purwin2008theory}
Purwin, O., D’Andrea, R., \BBA\ Lee, J.-W. \BBOP2008\BBCP.
\newblock \BBOQ Theory and implementation of path planning by negotiation for decentralized agents\BBCQ\
\newblock {\Bem Robotics and Autonomous Systems}, {\Bem 56\/}(5), 422--436.

\bibitem[\protect\BCAY{Rodriguez\ \BBA\ Amato}{Rodriguez\ \BBA\ Amato}{2010}]{ra-bbep-10}
Rodriguez, S.\BBACOMMA\  \BBA\ Amato, N.~M. \BBOP2010\BBCP.
\newblock \BBOQ Behavior-based evacuation planning\BBCQ\
\newblock In {\Bem 2010 IEEE International Conference on Robotics and Automation}, \BPGS\ 350--355. IEEE.

\bibitem[\protect\BCAY{Rus, Donald,\ \BBA\ Jennings}{Rus et~al.}{1995}]{rdjj-mfwtofar-95}
Rus, D., Donald, B., \BBA\ Jennings, J. \BBOP1995\BBCP.
\newblock \BBOQ Moving furniture with teams of autonomous robots\BBCQ\
\newblock In {\Bem Proceedings 1995 IEEE/RSJ International Conference on Intelligent Robots and Systems. Human Robot Interaction and Cooperative Robots}, \lowercase{\BVOL}~1, \BPGS\ 235--242. IEEE.

\bibitem[\protect\BCAY{Salzman\ \BBA\ Stern}{Salzman\ \BBA\ Stern}{2020}]{salzman2020research}
Salzman, O.\BBACOMMA\  \BBA\ Stern, R. \BBOP2020\BBCP.
\newblock \BBOQ Research challenges and opportunities in multi-agent path finding and multi-agent pickup and delivery problems\BBCQ\
\newblock In {\Bem Proceedings of the 19th International Conference on Autonomous Agents and MultiAgent Systems}, \BPGS\ 1711--1715.

\bibitem[\protect\BCAY{Sharon, Stern, Felner,\ \BBA\ Sturtevant}{Sharon et~al.}{2015}]{cbs}
Sharon, G., Stern, R., Felner, A., \BBA\ Sturtevant, N.~R. \BBOP2015\BBCP.
\newblock \BBOQ Conflict-based search for optimal multi-agent pathfinding\BBCQ\
\newblock {\Bem Artificial Intelligence}, {\Bem 219}, 40--66.

\bibitem[\protect\BCAY{Shim, Kim,\ \BBA\ Sastry}{Shim et~al.}{2003}]{shim2003decentralized}
Shim, D.~H., Kim, H.~J., \BBA\ Sastry, S. \BBOP2003\BBCP.
\newblock \BBOQ Decentralized nonlinear model predictive control of multiple flying robots\BBCQ\
\newblock In {\Bem 42nd IEEE International Conference on Decision and Control (IEEE Cat. No. 03CH37475)}, \lowercase{\BVOL}~4, \BPGS\ 3621--3626. IEEE.

\bibitem[\protect\BCAY{Sigurd\ \BBA\ How}{Sigurd\ \BBA\ How}{2003}]{sigurd2003uav}
Sigurd, K.\BBACOMMA\  \BBA\ How, J. \BBOP2003\BBCP.
\newblock \BBOQ Uav trajectory design using total field collision avoidance\BBCQ\
\newblock In {\Bem AIAA Guidance, Navigation, and Control Conference and Exhibit}, \BPG\ 5728.

\bibitem[\protect\BCAY{Silver}{Silver}{2005}]{s-cp-05}
Silver, D. \BBOP2005\BBCP.
\newblock \BBOQ Cooperative pathfinding\BBCQ\
\newblock In {\Bem Proceedings of the aaai conference on artificial intelligence and interactive digital entertainment}, \lowercase{\BVOL}~1, \BPGS\ 117--122.

\bibitem[\protect\BCAY{Solis, Motes, Sandstr{\"o}m,\ \BBA\ Amato}{Solis et~al.}{2021}]{smsa-romrmpucbs-21}
Solis, I., Motes, J., Sandstr{\"o}m, R., \BBA\ Amato, N.~M. \BBOP2021\BBCP.
\newblock \BBOQ Representation-optimal multi-robot motion planning using conflict-based search\BBCQ\
\newblock {\Bem IEEE Robotics and Automation Letters}, {\Bem 6\/}(3), 4608--4615.

\bibitem[\protect\BCAY{Stern, Sturtevant, Felner, Koenig, Ma, Walker, Li, Atzmon, Cohen, Kumar, et~al.}{Stern et~al.}{2019}]{mapf-benchmarks}
Stern, R., Sturtevant, N., Felner, A., Koenig, S., Ma, H., Walker, T., Li, J., Atzmon, D., Cohen, L., Kumar, T., et~al. \BBOP2019\BBCP.
\newblock \BBOQ Multi-agent pathfinding: Definitions, variants, and benchmarks\BBCQ\
\newblock In {\Bem Proceedings of the International Symposium on Combinatorial Search}, \lowercase{\BVOL}~10, \BPGS\ 151--158.

\bibitem[\protect\BCAY{Tanner, Pappas,\ \BBA\ Kumar}{Tanner et~al.}{2004}]{tpk-ltfs-04}
Tanner, H.~G., Pappas, G.~J., \BBA\ Kumar, V. \BBOP2004\BBCP.
\newblock \BBOQ Leader-to-formation stability\BBCQ\
\newblock {\Bem IEEE Transactions on robotics and automation}, {\Bem 20\/}(3), 443--455.

\bibitem[\protect\BCAY{Velagapudi, Sycara,\ \BBA\ Scerri}{Velagapudi et~al.}{2010}]{velagapudi2010decentralized}
Velagapudi, P., Sycara, K., \BBA\ Scerri, P. \BBOP2010\BBCP.
\newblock \BBOQ Decentralized prioritized planning in large multirobot teams\BBCQ\
\newblock In {\Bem 2010 IEEE/RSJ International Conference on Intelligent Robots and Systems}, \BPGS\ 4603--4609. IEEE.

\bibitem[\protect\BCAY{Wagner\ \BBA\ Choset}{Wagner\ \BBA\ Choset}{2015}]{wc-sefmrpp-15}
Wagner, G.\BBACOMMA\  \BBA\ Choset, H. \BBOP2015\BBCP.
\newblock \BBOQ Subdimensional expansion for multirobot path planning\BBCQ\
\newblock {\Bem Artificial intelligence}, {\Bem 219}, 1--24.

\bibitem[\protect\BCAY{Wang, Liu, Qiu,\ \BBA\ Zhou}{Wang et~al.}{2022}]{wang2022consensus}
Wang, S., Liu, Y., Qiu, Y., \BBA\ Zhou, J. \BBOP2022\BBCP.
\newblock \BBOQ Consensus-based decentralized task allocation for multi-agent systems and simultaneous multi-agent tasks\BBCQ\
\newblock {\Bem IEEE Robotics and Automation Letters}, {\Bem 7\/}(4), 12593--12600.

\bibitem[\protect\BCAY{Wilt\ \BBA\ Botea}{Wilt\ \BBA\ Botea}{2014}]{wilt2014spatially}
Wilt, C.\BBACOMMA\  \BBA\ Botea, A. \BBOP2014\BBCP.
\newblock \BBOQ Spatially distributed multiagent path planning\BBCQ\
\newblock In {\Bem Proceedings of the International Conference on Automated Planning and Scheduling}, \lowercase{\BVOL}~24, \BPGS\ 332--340.

\bibitem[\protect\BCAY{Xie, Hu, Bhowmick, Ding,\ \BBA\ Arvin}{Xie et~al.}{2022}]{xie2022distributed}
Xie, S., Hu, J., Bhowmick, P., Ding, Z., \BBA\ Arvin, F. \BBOP2022\BBCP.
\newblock \BBOQ Distributed motion planning for safe autonomous vehicle overtaking via artificial potential field\BBCQ\
\newblock {\Bem IEEE Transactions on Intelligent Transportation Systems}, {\Bem 23\/}(11), 21531--21547.

\bibitem[\protect\BCAY{Yu\ \BBA\ LaValle}{Yu\ \BBA\ LaValle}{2013}]{yl-saioomrppog-13}
Yu, J.\BBACOMMA\  \BBA\ LaValle, S. \BBOP2013\BBCP.
\newblock \BBOQ Structure and intractability of optimal multi-robot path planning on graphs\BBCQ\
\newblock In {\Bem Proceedings of the AAAI Conference on Artificial Intelligence}, \lowercase{\BVOL}~27, \BPGS\ 1443--1449.

\bibitem[\protect\BCAY{Yu\ \BBA\ Rus}{Yu\ \BBA\ Rus}{2015}]{yu2015pebble}
Yu, J.\BBACOMMA\  \BBA\ Rus, D. \BBOP2015\BBCP.
\newblock \BBOQ Pebble motion on graphs with rotations: Efficient feasibility tests and planning algorithms\BBCQ\
\newblock In {\Bem Algorithmic Foundations of Robotics XI: Selected Contributions of the Eleventh International Workshop on the Algorithmic Foundations of Robotics}, \BPGS\ 729--746. Springer.

\end{thebibliography}
\bibliographystyle{theapa}

\end{document}